\documentclass[11pt,twoside]{article}
\usepackage{fullpage}
\usepackage{epsf}
\usepackage{fancyheadings}
\usepackage{hyperref}
\usepackage{url}

\usepackage{pstricks,pst-node,pst-text,pst-3d}
\usepackage{color}

 \usepackage{amsthm}
 \usepackage{amsfonts}
 \usepackage{amsmath}
 \usepackage{amssymb}
\usepackage{bbm}

\usepackage{float}
\floatstyle{ruled}
\newfloat{procedure}{thp}{lop}
\floatname{procedure}{Procedure}

\newfloat{algorithm}{thp}{lop}
\floatname{algorithm}{Algorithm}

\usepackage{enumerate}

\usepackage{graphics}
\usepackage{graphicx}
\usepackage{psfrag}
\usepackage{pgf,tikz}
\usepackage{mathrsfs}
\usetikzlibrary{arrows}

\usepackage{epsf}
 \usepackage{subfigure}
 \usepackage{caption}
\usepackage[titletoc,title]{appendix}

\newcommand{\bcar}{\begin{carlist}}
\newcommand{\ecar}{\end{carlist}} 

\newlength{\widebarargwidth}
\newlength{\widebarargheight}
\newlength{\widebarargdepth}

\makeatletter
\long\def\@makecaption#1#2{
        \vskip 0.8ex
        \setbox\@tempboxa\hbox{\small {\bf #1:} #2}
        \parindent 1.5em  
        \dimen0=\hsize
        \advance\dimen0 by -3em
        \ifdim \wd\@tempboxa >\dimen0
                \hbox to \hsize{
                        \parindent 0em
                        \hfil 
                        \parbox{\dimen0}{\def\baselinestretch{0.96}\small
                                {\bf #1.} #2
                                } 
                        \hfil}
        \else \hbox to \hsize{\hfil \box\@tempboxa \hfil}
        \fi
        }
\makeatother

\newenvironment{carlist}
 {\begin{list}{$\bullet$}
 {\setlength{\topsep}{0in} \setlength{\partopsep}{0in}
  \setlength{\parsep}{0in} \setlength{\itemsep}{\parskip}
  \setlength{\leftmargin}{0.07in} \setlength{\rightmargin}{0.08in}
  \setlength{\listparindent}{0in} \setlength{\labelwidth}{0.08in}
  \setlength{\labelsep}{0.1in} \setlength{\itemindent}{0in}}}
 {\end{list}}

\theoremstyle{plain}

\newtheorem{theos}{Theorem}
\newtheorem{claim}{Claim}
\newtheorem{props}{Proposition}
\newtheorem{lems}{Lemma}

\theoremstyle{remark}

\theoremstyle{remark}

\theoremstyle{remark}

\long\def\comment#1{}

\newcommand{\widgraph}[2]{\includegraphics[keepaspectratio,width=#1]{#2}}

\def\NN{ \mathbb{N} }						
\def\eps{ \epsilon }
\def\EE{ \mathbb{E} }
\def\PP{ \mathbb{P} }
\def\E{ \mathrm{e} }	
\newcommand{\Indi}{\mathbbm{1}}

\newcommand{\barmu}{\bar{\mu}}

\newcommand{\numarms}{K}

\newcommand{\thenull}{H_0}

\newcommand{\noarms}{\numarms}
\newcommand{\nullp}[1]{P^{#1}}

\newcommand{\dev}{\varphi}
\newcommand{\banditstop}{T(\putinalpha)}

\newcommand{\rowstop}{T}
\newcommand{\trunctime}{M}

\newcommand{\empmu}{\widehat{\mu}}

\newcommand{\wlth}{W}
\newcommand{\sfield}{\mathcal{F}}
\newcommand{\FDR}{\text{FDR}}
\newcommand{\mFDR}{\text{mFDR}}
\newcommand{\BDR}{\text{BDR}}

\newcommand{\LCB}{\text{LCB}}
\newcommand{\UCB}{\text{UCB}}
\newcommand{\realbest}{i_{\star}}

\newcommand{\numexp}{J}

\newcommand{\truenulls}{\mathcal{H}_0}
\newcommand{\falsenulls}{\mathcal{H}_1}
\newcommand{\wealth}{W}

\newcommand{\rej}{R}
\newcommand{\defn}{:=}
\newcommand{\mumax}{\mu_{\realbest}}

\newcommand{\epsBDR}{\eps\BDR}
\newcommand{\MAB}{\text{MAB}}

\newcommand{\fABFDR}{AB-FDR}
\newcommand{\fMABFDR}{MAB-FDR}
\newcommand{\fMABIND}{MAB-IND}
\newcommand{\FDP}{\text{FDP}}
\newcommand{\Gap}{\Delta}

\newcommand{\putinalpha}{\alpha_j}
\newcommand{\tautil}{\widetilde{\tau}}
\newcommand{\UpperSample}{B}
\newcommand{\effgap}{\widetilde{\Gap}}

\newcommand{\SetStar}{\ensuremath{\mathcal{S}^{\star}}}
\newcommand{\Output}{i_b}
\newcommand{\event}{\ensuremath{\mathcal{E}}}
\newcommand{\tindex}{\ensuremath{s}}

\begin{document}


\begin{center}

  {\bf{\LARGE{A framework for Multi-A(rmed)/B(andit) testing \\
        with online FDR control}}}

\vspace*{.2in}

{\large{
\begin{tabular}{cccc}
Fanny Yang$^{\star}$ & Aaditya Ramdas$^{\dagger,\star}$ & Kevin Jamieson$^{\star}$ & Martin J. Wainwright$^{\dagger,\star}$ \\
\end{tabular}
}}

\vspace*{.2in}

\begin{tabular}{c}
Department of Statistics$^\dagger$, and \\
Department of Electrical Engineering and Computer Sciences$^\star$ \\
UC Berkeley,  Berkeley, CA  94720
\end{tabular}

\vspace*{.1in}


\begin{abstract}
    We propose an alternative framework to existing setups for
  controlling false alarms when multiple A/B tests are run over time.
  This setup arises in many practical applications, e.g. when
  pharmaceutical companies test new treatment options against control
  pills for different diseases, or when internet companies test their
  default webpages versus various alternatives over time.  Our
  framework proposes to replace a sequence of A/B tests by a sequence
  of best-arm MAB instances, which can be continuously monitored by
  the data scientist.  When interleaving the MAB tests with an an
  online false discovery rate (FDR) algorithm, we can obtain the best
  of both worlds: low sample complexity and any time online FDR
  control.  Our main contributions are: (i) to propose reasonable
  definitions of a null hypothesis for MAB instances; (ii) to
  demonstrate how one can derive an always-valid sequential $p$-value
  that allows continuous monitoring of each MAB test; and (iii) to
  show that using rejection thresholds of online-FDR algorithms as the
  confidence levels for the MAB algorithms results in both
  sample-optimality, high power and low FDR at any point in time.  We
  run extensive simulations to verify our claims, and also report
  results on real data collected from the New Yorker Cartoon Caption
  contest.
\end{abstract}

\end{center}

\section{Introduction}

For most modern internet companies, wherever there is a metric that
can be measured (e.g., time spent on a page, click-through rates,
conversion of curiousity to a sale), there is almost always a
randomized trial behind the scenes, with the goal of identifying 
an alternative website design that provides improvements over the default
design.  The use of such data-driven decisions for perpetual
improvement is colloquially known as \emph{A/B testing} in the case of
two alternatives, or \emph{A/B/n testing} for several alternatives.
Given a default configuration and several alternatives (e.g., color
schemes of a website), the standard practice is to divert a small
amount of scientist-traffic to a randomized trial over these alternatives
and record the desired metric for each of them.
If an alternative appears to be significantly better, it is
implemented; otherwise, the default setting is maintained. 

At first
glance, this procedure seems intuitive and simple.
However, in cases where the aim is to optimize over one particular
metric, this common tool suffers from several downsides. (1) First,
whereas some alternatives may be clearly worse than the default,
others may only have a slight edge.  If one wishes to minimize the
amount of time and resources spent on this randomized trial
the more promising alternatives should intuitively get a larger share
of the traffic than the clearly-worse alternatives.  Yet typical
A/B/n testing frameworks allocate traffic uniformly over alternatives.
(2) Second, companies often desire to continuously monitor an ongoing
A/B test as they may adjust their termination criteria as time goes by
and possibly stop earlier or later than originally intended.  However,
just as if you flip a coin long enough, a long string of heads is
eventually inevitable, the practice of continuous monitoring (without
mathematically correcting for it) can easily fool the tester to
believe that a result is statistically significant, when in reality it
is not. This is one of the reasons for the lack of
reproducibility of scientific results, an issue recently receiving
increased attention from the public media.
%
(3) Third, the lack of sufficient evidence or an insignificant
improvement of the metric may make it undesirable from a practical or
financial perspective to replace the default.  Therefore, when a
company runs hundreds to thousands of A/B tests within a year, ideally
the number of statistically insignificant changes that it made should
be small compared to the total number of changes made.  Controlling
the false alarm rate of each individual test at a desired level
$\alpha$ however does \emph{not} achieve this type of control, also
known as controlling the false discovery rate.  Of course, it is also
desirable to detect better alternatives (when they exist), and to do
so as quickly as possible. 

In this paper, we provide a novel framework that addresses the above
shortcomings of A/B or A/B/n testing.  The first concern is tackled by
employing recent advances in adaptive sampling like the
pure-exploration multi-armed bandit (MAB) algorithm.  For the second
concern, we adopt the notion of any-time $p$-values for guilt-free
continuous monitoring, and we make the advantages and risks of
early-stopping transparent.  Finally, we handle the third issue using
recent advances in online false discovery rate (FDR) control. Hence
the combined framework can be described as doubly-sequential
(sequences of MAB tests, each of which is itself sequential).
Although each of those problems has been studied in hitherto disparate
communities, how to leverage the best of all worlds, if at all
possible, has remained an open problem.  The main contributions of
this paper are in merging these ideas in a combined framework and
presenting the conditions under which it can be shown to yield
near-optimal sample complexity, near-optimal best-alternative
discovery rate, as well as FDR control.

While the above concerns raised about A/B/n testing were discussed
using the example of modern internet companies, the same
concerns carry forward qualitatively to other domains, like
pharmaceutical companies running sequential clinical trials with a
control (often placebo) and a few treatments (like different doses or
drug substances).  In a manufacturing or food production setting, one
may be interested in identifying (perhaps cheaper) substitutes for 
individual materials without compromising the quality of a product too
much.  In a government setting, pilot programs are funded in
search of improvements over current programs and it is desirable
from a  social welfare standpoint and cost to limit the adoption
of ineffective policies.

The remainder of this paper is organized as follows.  In
Section~\ref{SecGoals}, we lay out the primary goals of the paper, and
describe a meta-algorithm that combines adaptive sampling strategies
with FDR control procedures.  Section~\ref{SecConcrete} is devoted to
the description of a concrete procedure, along with some theoretical
guarantees on its properties.  In Section~\ref{SecExperiments}, we
describe the results of our extensive experiments on both simulated
and real-world data sets that are available to us, before we conclude with a discussion in
Section~\ref{SecDiscussion}.

  \vspace{-0.1in}
\section{Formal experimental setup and a meta-algorithm}
\label{SecGoals}

In this section we first formalize the setup of a typical A/B/n test
and provide a high-level overview of our proposed combined framework
aimed at addressing the shortcomings mentioned in the
introduction. A specific instantiation of this meta-algorithm along
with detailed theoretical guarantees are
specified in Section~\ref{SecConcrete}.

For concreteness, we refer to the system designer, whether a tech
company or a pharmaceutical company, as a (data) scientist.  We assume that the
scientist needs to possibly conduct an infinite number of experiments
sequentially, indexed by $j$.  Each experiment has one default
setting, referred to as the \emph{control}, and $\numarms =
\numarms(j)$ alternative settings, called the \emph{treatments} or
\emph{alternatives.}  The scientist must return one of the
$\numarms + 1$ options that is the ``best'' according to some
predefined metric, before the next experiment is started. Such a setup is
a simple mathematical model both for clinical trials run by
pharmaceutical labs, and A/B/n testing used at scale by tech
companies.

One full experiment consists of steps of the following kind: In each
step, the scientist assigns a new person---who arrives at the
website or who enrolls in the clinical trial---to one of the $\numarms
+1$ options and obtains a measurable outcome.  In practice, the role of the
scientist could be taken by an adaptive algorithm, which determines
the assignment at time step $j$ by careful consideration of all
previous outcomes.  Borrowing terminology from the multi-armed bandit
(MAB) literature, we refer to each of the $\numarms+1$ options as an
\emph{arm}, and each assignment to arm $i$ is termed ``pulling arm
$i$''.  For concreteness, we assign the index $0$ to the default or
control arm, and note that this index is known to the algorithm.

We assume that the observable metric from each pull of arm
$i=0,1,\dots,\numarms$ corresponds to an independent draw from an
unknown probability distribution with expectation $\mu_i$. Ideally, if
the means were known, we would use them as scores to compare the arms
where higher is better. In the sequel we use $\mumax \defn \max
\limits_{i=1,\dots, \numarms} \mu_i$ to denote the mean of the best
arm. 
We refer the reader to Table~\ref{tab:notation} for a glossary of the
notation used throughout this paper.

\vspace{-0.04in}
\subsection{Some desiderata and difficulties}

Given the setup above, how can we mathematically describe the
guarantees that the companies might desire from an improved
multiple-A/B/n testing framework?  Which parts of the puzzle can be
directly transferred from known results, and what challenges remain?

In order to answer the first question, let us adopt terminology from
the hypothesis testing literature and view each experiment as a test
of a \emph{null hypothesis}.  Any claim that an alternative arm is the
best is called a \emph{discovery}, and if such a claim is erroneous
then it is called a false discovery.
When multiple hypotheses need to be tested, the scientist needs
to define the quantity it wants to control. 
While we may desire that the probability of even a
single false discovery---called the family-wise error rate---is small,
this is usually far too stringent for a large and unknown number of
tests. For this reason, \cite{BH95} proposed that it may be more
interesting to control the expected ratio of false discoveries to the
total number of discoveries (called the False Discovery Rate, or
\emph{FDR} for short) or ratio of expected number of false discoveries
to the expected number of total discoveries (called the modified FDR
or \emph{mFDR} for short). 
Over the past decades, the FDR and its variants like mFDR have
become standard quantities for multiple testing applications. In the
following, if not otherwise specified, we use the term FDR to denote
both measures in order to simplify the presentation. In
Section~\ref{SecConcrete}, we show that
both mFDR and FDR can be controlled for different choices of
procedures.

\vspace{-0.1in}
\subsubsection{Challenges in viewing an MAB instance as a hypothesis test}

In our setup, we want to be able to control the FDR at any time in an
online manner. Online FDR procedures were first introduced by Foster
and Stine~\cite{FS08}, and have since been studied by other authors
(e.g.,~\cite{AR14,JM16}). A typical online FDR procedure is based on
comparing a valid $p$-value $\nullp{j}$ with carefully-chosen levels
$\alpha_j$ for each hypothesis test\footnote{A valid $\nullp{j}$ must
  be stochastically dominated by a uniform distribution on $[0,1]$,
  which we henceforth refer to as \emph{super-uniformly distributed}.}.
We reject the null hypothesis, represented as $R_j = 1$, when
$\nullp{j} \leq \alpha_j$ and we set $R_j=0$ otherwise.
%
%

As mentioned, we want to use adaptive MAB algorithms in each
experiment to test each hypothesis, since they can find a best arm among
$K+1$ with near-optimal sample complexity. However the traditional MAB
setup does not account for the asymmetry between the arms as is the
case in a testing setup, with one being the default (control) and
others being alternatives (treatments). This is the standard scenario
in A/B/n testing applications, as for example a company might prefer
wrong claims that the control is the best (false negative), rather
than wrong claims that an alternative is the best (false positive),
simply because new system-wide adoption of selected alternatives might
involve high costs.  What would be a suitable null hypothesis in this
hybrid setting?  To allow continuous monitoring, is it possible to define 
and compute always-valid $p$-values
that are super-uniformly distributed under the null hypothesis
when computed at any time $t$?
(This could be especially challenging  given that the number of samples 
from each the arm is random, and different for each arm.)

In addition to asymmetry, the practical scientist might have a different
incentive than the ideal outcome for MAB algorithms. In particular,
he/she might not want to find the best alternative if it is
not \emph{substantially} better than the control. Indeed, if the net
gain made by adopting a new alternative is small, it might be offset
by the cost of implementing the change from the existing default
choice.  By similar reasoning, we may not require identifying the
single best arm if there is a \emph{set} of arms with similar means
that are all larger than the rest.  

We propose a sensible null-hypothesis for each experiment which
incorporates the approximation and improvement notions as
described above and provide an always valid $p$-value which can be
easily calculated at each time step in the experiment. We show that a
slight modification of the usual LUCB algorithm caters to this
specific null-hypothesis while still maintaining near-optimal sample
complexity.

\vspace{-0.05in}
\subsubsection{Interaction between MAB and FDR}

In order to take advantage of the sample efficiency of best-arm bandit
algorithms, it is crucial to set the confidence levels close to what
is needed. Given a user-defined level $\alpha$, at each hypothesis
$j$, online FDR procedures automatically output the significance level
$\alpha_j$ which are ``needed'' to guarantee FDR control, based on
past decisions.

\begin{figure}[hbtp]
  \begin{center}
    \includegraphics[width=0.9\textwidth]{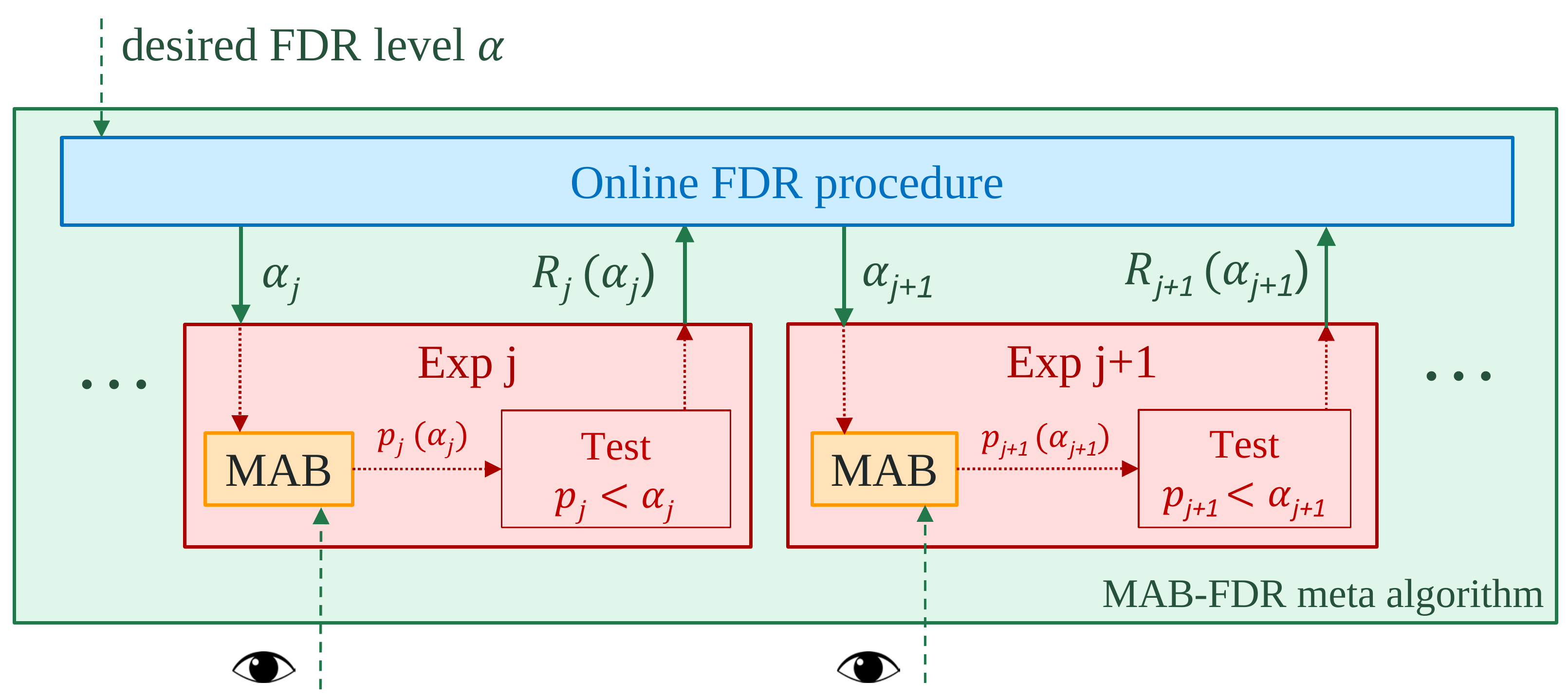}
    \end{center}
  \caption{Diagram of the MAB-FDR meta algorithm designed to achieve
    online FDR control along with near-optimal sample complexity. The
    green arrows symbolize interaction between the MAB and FDR
    procedures via the FDR test levels $\alpha_j$ and rejection
    indicator variables $R_j$. Notice that the $\nullp{j}$-values are
    now dependent as each $\alpha_j$ depends on $R_1, \dots,
    R_{j-1}$. The eyes represent possible continuous monitoring by the
    scientist.}
  \label{FigDiagram}
\end{figure}

Can we directly set the MAB confidence levels to the
output levels $\alpha_j$ from the online FDR procedure?  If we do, 
our $p$-values are not independent across different
hypotheses anymore: $\nullp{j}$ directly depends on the FDR levels $\alpha_j$
and each $\alpha_j$ in turn depends on past MAB rejections, thus on
past MAB $p$-values (see Figure~\ref{FigDiagram}). 
Does the new interaction compromise FDR guarantees?

Although known online FDR procedures~\cite{FS08,JM16} guarantee FDR
control for independent $p$-values, this does not hold for dependent
$p$-values in general. Hence FDR control guarantees cannot simply be
obtained out of the box. In particular, it is not a priori obvious
that the introduced dependence between the $p$-values does not cause
problems, i.e. violates necessary conditions for FDR control type
theorems.  A key insight that emerges from our analysis is that an
appropriate bandit algorithm actually shapes the $p$-value
distribution under the null in a “good” way that allows us to control
FDR.

\subsection{A meta-algorithm}
\label{SecCombineHypoMAB}

Procedure~\ref{ProcFramework} summarizes our doubly-sequential
procedure, with a corresponding flowchart in
Figure~\ref{FigDiagram}. We will prove theoretical guarantees after
instantiating the separate modules.  Note that our framework allows
the scientist to plug in their favorite best-arm MAB algorithm or
online FDR procedure. The choice for each of them determines which
guarantees can be proven for the entire setup. Any independent
improvement in either of the two parts would immediately lead to an
overall performance boost of the overall framework.

\vspace{-0.01in}
\begin{procedure}[h!]
\caption{MAB-FDR Meta algorithm skeleton}
\label{ProcFramework}
\begin{enumerate}
\item The scientist sets a desired FDR control rate $\alpha$.
\item For each $j = 1, 2, \dots$:
  \bcar
  \item Experiment $j$ receives a designated control arm and some
    number of alternative arms.
       \item An \emph{online-FDR procedure} returns an $\alpha_{j}$
         that is some function of the past values
         $\{P^{\ell}\}_{\ell=1}^{j-1}$.
    \item An \emph{MAB procedure} with inputs (a) the control arm and
      $K(j)$ alternative arms, (b) confidence level $\alpha_j$, and
      (c) (optional) a precision $\epsilon \geq 0$, is executed and if
      the procedure self-terminates, returns a recommended arm.
    \item Throughout the MAB procedure, an \emph{always valid
      $p$-value} is constructed continuously for each time $t$ using
      only the samples collected up to that time from the $j$-th
      experiment: for any $t$, it is a random variable $P^j_t \in
      [0,1]$ that is super-uniformly distributed whenever the
      control-arm is best.
    \item When the MAB procedure is terminated at time $t$ (either by
      itself or by a user-defined stopping criterion that may depend on
      $P^j_t$), if the arm with the highest empirical mean is
      \emph{not} the control arm and $P^j_t \leq \alpha_j$, then we
      return $P^j := P^j_t$, and the control arm
      is rejected in favor of this empirically best arm.
\ecar
\end{enumerate}
\vspace{-0.1in}
\end{procedure}


\section{A concrete procedure with guarantees}
\label{SecConcrete}

We now take the high-level road map given in
Procedure~\ref{ProcFramework}, and show that we can obtain a concrete,
practically implementable framework with FDR control and power
guarantees.  We first discuss the key modeling decisions we have to
make in order to seamlessly embed MAB algorithms into an online FDR
framework.  We then outline a modified version of a commonly used
best-arm algorithm, before we finally prove FDR and power guarantees
for the concrete combined procedure.

\subsection{Defining null hypotheses and constructing $p$-values} 
\label{SecNullhypPval}

Our first task is to define a null hypothesis for each experiment.  As
mentioned before, the choice of the null is not immediately obvious,
since we sample from \emph{multiple} distributions \emph{adaptively}
instead of independently.  In particular, we will generally not have
the same number of samples for all arms.
Given a distribution with default mean $\mu_0$ and alternative
distributions with means $\{\mu_i \}_{i=1}^K$, we propose that the
null hypothesis for the $j$-th experiment should be defined as
\begin{align}
\label{EqNullHyp}
\thenull^j: \mu_0 \geq \mu_i - \epsilon \quad \mbox{for all $i =
  1,\dots,\noarms$.}
\end{align}
In words, the null
corresponds to there being no alternative arm that is
$\epsilon$-better than the control arm.

It remains to define a $p$-value for each experiment that is
stochastically dominated by a uniform random variable under the null;
such a $p$-value is said to be \emph{superuniform}.  In order to simplify
notation below, we omit the index $j$ for the experiment and retain
only the index $i$ for the choice of arms.  In order to be able to use
a $p$-value at arbitrary times in the testing procedure and to allow
scientists to monitor the algorithm's progress in real time, it is
helpful to define an \emph{always valid $p$-value}, as previously
defined by Johari et al.~\cite{JPW15}. An always valid p-value is a
stochastic process $\{P_t\}_{t=1}^\infty$ such that for all fixed and
random stopping times $T$, under any distribution $\PP_0$ over the arm
rewards such that the null hypothesis is true, we have
\begin{align}
\label{EqnSuperuniPvalue}
\PP_0(P_T \leq \alpha) & \leq \alpha.
\end{align}
When all arms are drawn independently an equal number of times, 
by linearity of expectation one can regard the distance of each pair
of samples as a random variable drawn i.i.d. from a distribution with
mean $\tilde{\mu}_i\defn\mu_0 - \mu_i$. We can then view the problem as
testing the standard hypothesis $\thenull: \tilde{\mu}_i > -\epsilon$.
However, when the arms are pulled adaptively, a different solution
needs to be found---indeed, in this case, the sample means are
\emph{not unbiased estimators} of the true means, since the number of
times an arm was pulled now depends on the empirical means of all the
arms.

Our strategy is to construct always valid $p$-values by using the fact
that p-values can be obtained by inverting confidence intervals.  To
construct always-valid confidence bounds, we resort to the fundamental
concept of the law of the iterated logarithm (LIL), for which
non-asymptotic versions have been recently derived and used for both
bandits and testing problems (see \cite{JMNB14}, \cite{BR16}).

To elaborate, define the function
\begin{align}
\label{EqnDevDef}
\dev_{n}(\delta) = \sqrt{\frac{\log(\frac{1}{\delta}) +
    3\log(\log(\frac{1}{\delta})) + \frac{3}{2} \log(\log(\E n))}{n}}.
\end{align}
If $\widehat{\mu}_{i,n}$ is the empirical average of independent
samples from a sub-Gaussian distribution, then it is known (see, for
instance, ~\cite[Theorem 8]{KCG15}) that for all $\delta
\in (0,1)$, we have
\begin{align}
\label{eqn:lil_confidence}  
\max \Big\{ &\PP\Big( \bigcup_{n=1}^\infty \{ \widehat{\mu}_{i,n} -
\mu_i > \dev_n(\delta \wedge 0.1) \} \Big), \quad \PP\Big(
\bigcup_{n=1}^\infty \{ \widehat{\mu}_{i,n} - \mu_i < -\dev_n(\delta
\wedge 0.1) \} \Big) \Big\} \leq \delta,
\end{align}
where $\delta \wedge 0.1 \defn \min \{ \delta, 0.1 \}$.

We are now ready to propose single arm $p$-values of the form
\begin{align}
\label{EqnMultiP}
P_{i,t} :&= \sup \Big\{ \gamma \in [0,1] \; \mid \;
\:\widehat{\mu}_{i,n_i(t)} - \dev_{n_i(t)}(\tfrac{\gamma}{2\numarms})
\leq \: \widehat{\mu}_{0,n_0(t)} +\dev_{n_0(t)}(\tfrac{\gamma}{2}) +
\epsilon \Big \}\\
& = \sup \Big\{ \gamma \in [0,1] \; \mid \; \LCB_i(t) \leq \UCB_0(t) + \epsilon\Big\} \nonumber
\end{align}
Here we set $P_{i,t}=1$ if the supremum is taken over an empty set.
Given these single arm $p$-values, the always-valid $p$-value for the
experiment is defined as
\begin{align}
\label{EqnPVal}
P_t & \defn \min_{\tindex \leq t} \; \min_{i= 1, \ldots, \numarms}
P_{i,\tindex}.
\end{align}
We claim that this
procedure leads to an always valid $p$-value (with proof in Appendix~\ref{SecProofPropPVal}).
\begin{props}
\label{PropPVal}
The sequence $\{P_t \}_{t=1}^\infty$ defined via
equation~\eqref{EqnPVal} is an always valid $p$-value.
\end{props}
\noindent 
See Section~\ref{SecProofPropPVal} for the proof of this proposition.


\subsection{Adaptive sampling for best-arm identification}
\label{SecBestarm}

In the traditional A/B testing setting described in the introduction,
samples are allocated uniformly to the different alternatives.  But by
allocating different numbers of samples to the alternatives,
decisions can be made with the same statistical significance using far fewer
samples.
Suppose moreover that there is a unique maximizer
$\realbest \defn \arg \max \limits_{i=0,1,\dots,K} \mu_i$, so that
\begin{align*}
  \Delta_i \defn \mu_{\realbest} - \mu_i > 0 \qquad \mbox{for all $i
    \neq \realbest$.}
\end{align*}
Then for any $\delta \in (0,1)$, best-arm identification algorithms
for the multi-armed bandit problem can identify $\realbest$ with
probability at least $1-\delta$ based on at most\footnote{Here we have
  ignored some doubly-logarithmic factors.}
$\sum_{i \neq \realbest} \Delta_i^{-2} \log(1/\delta)$ total samples
(see the paper~\cite{jamieson2014best} for a brief survey and
\cite{villar15} for an application to clinical trials).  In contrast,
if samples are allocated \emph{uniformly} to the alternatives under
the same conditions, then the most natural procedures require
$K \max \limits_{i \neq \realbest} \Delta_i^{-2} \log(K/\delta)$
samples before returning $\realbest$ with probability at least
$1-\delta$.

However, standard best-arm bandit algorithms do not incorporate
asymmetry as induced by null-hypotheses as in
definition~\eqref{EqNullHyp} by default.  Furthermore, recall that a
practical scientist might desire the ability to incorporate
approximation and a minimum improvement requirement. More precisely,
it is natural to consider the requirement that the returned arm
$\Output$ satisfies the bounds $\mu_{\Output} \geq \mu_0 + \eps$ and
$\mu_{\Output} \geq \mu_{\realbest} - \eps$ for some $\eps > 0$.  For
those readers unfamiliar with best-arm MAB algorithms, it is likely
helpful to first grasp the entire framework in the special $\eps = 0$
throughout, before understanding it in full generality with the
complications introduced by setting $\eps>0$.  In the following we
present a modified MAB algorithm based on the common LUCB algorithm
(see~\cite{KTAS12,SJR17}).


\begin{algorithm}[ht!]
\caption{Best-arm identification with a control arm for confidence
  $\delta$ and precision $\epsilon \geq 0$}
\label{AlgoModLUCB}
For all $t$ let $n_i(t)$ be the number of times arm $i$ has been
pulled up to time $t$. In addition, for each arm $i$ let $\empmu_i(t)
= \frac{1}{n_i(t)}\sum_{\tau = 1}^{n_i(t)} r_i(\tau)$, define
\vspace{-0.03in}
\begin{align*}
\LCB_i(t) &:= \widehat{\mu}_{i,n_i(t)} -
\dev_{n_i(t)}(\tfrac{\delta}{2K}) \qquad \mbox{and} \qquad \UCB_i(t) :=
\widehat{\mu}_{i,n_i(t)} + \dev_{n_i(t)}(\tfrac{\delta}{2}).
\end{align*} 
\begin{enumerate}
\item Set $ t=1$ and sample every arm once.
\item 
   Repeat: Compute $h_t = \arg \max \limits_{i=0,1,\dots,K}
   \widehat{\mu}_i(t)$, and $\ell_t = \arg \max
   \limits_{i=0,1,\dots,K,  i \neq h_t} \UCB_i(t)$

  \begin{enumerate}[(a)]
  \item  If $\LCB_{0}(t) > \UCB_{i}(t) - \eps$, for all $i \neq 0$, then output $0$ and terminate.\\ 
    Else if $\LCB_{h_t}(t) > \UCB_{\ell_t}(t) - \eps$ and
    $\LCB_{h_t}(t) > \UCB_{0}(t) + \eps$, then output
    $h_t$ and terminate.
  \item If $\epsilon >0$, let $u_t = \arg \max_{i\neq 0} \UCB_i(t)$ and pull all distinct arms in $\{0, u_t, h_t, \ell_t\}$ once. \\
      If $\epsilon = 0$, pull arms $h_t$ and $\ell_t$ and set $t= t+1$.

  \end{enumerate}
\end{enumerate}
\end{algorithm}


Inside the loop of Algorithm~\ref{AlgoModLUCB}, we use $h_t \in \{0,1,\dots,K\}$
to denote the current empirically-best arm, $\ell_t$ to denote the
most promising contender among the other arms that has not yet been
sampled enough to be ruled out.  The parameter $\epsilon\geq 0$ is a
slack variable, and the algorithm is easiest to first understand when
$\epsilon =0$.  We provide a visualization of how $\epsilon$ affects
the stopping condition in Figure~\ref{fig:LUCBEps}.  Step (a) checks
if $h_t$ is within $\epsilon$ of the true highest mean, and if it is
also at least $\epsilon$ greater than the true mean of the control arm
(or is the control arm), terminates with this arm $h_t$.  Step (b)
ensures that the control arm is sufficiently sampled when $\epsilon >
0$.  Step (c) pulls $h_t$ and $\ell_t$, reducing the overall
uncertainty in the difference between their two means.  


\begin{figure}[b!]
  \begin{tabular}{c c}
  \includegraphics[width=0.42\textwidth,trim={0 0 0 0},clip]{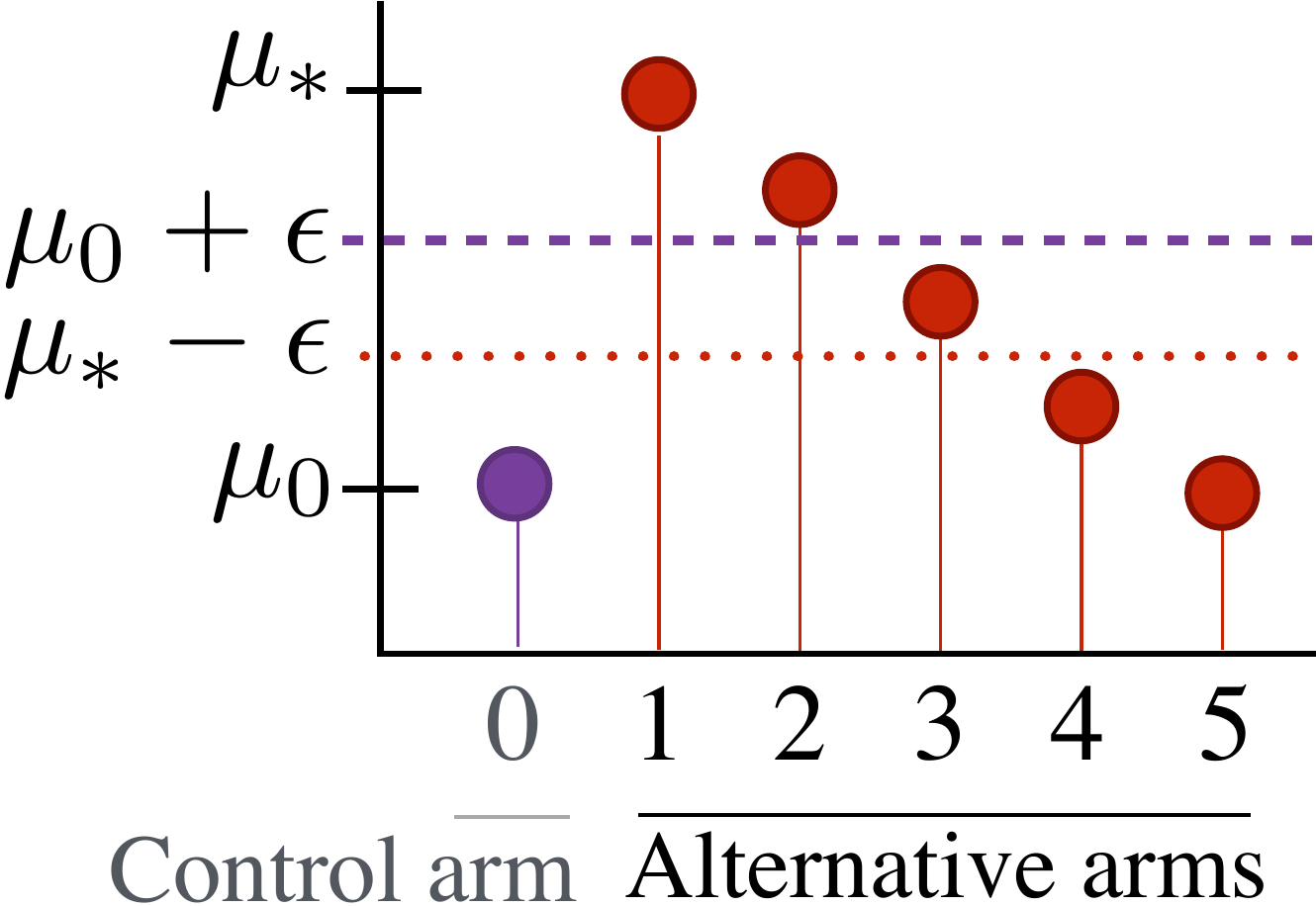} & \hspace{.5cm}
  \includegraphics[width=0.42\textwidth]{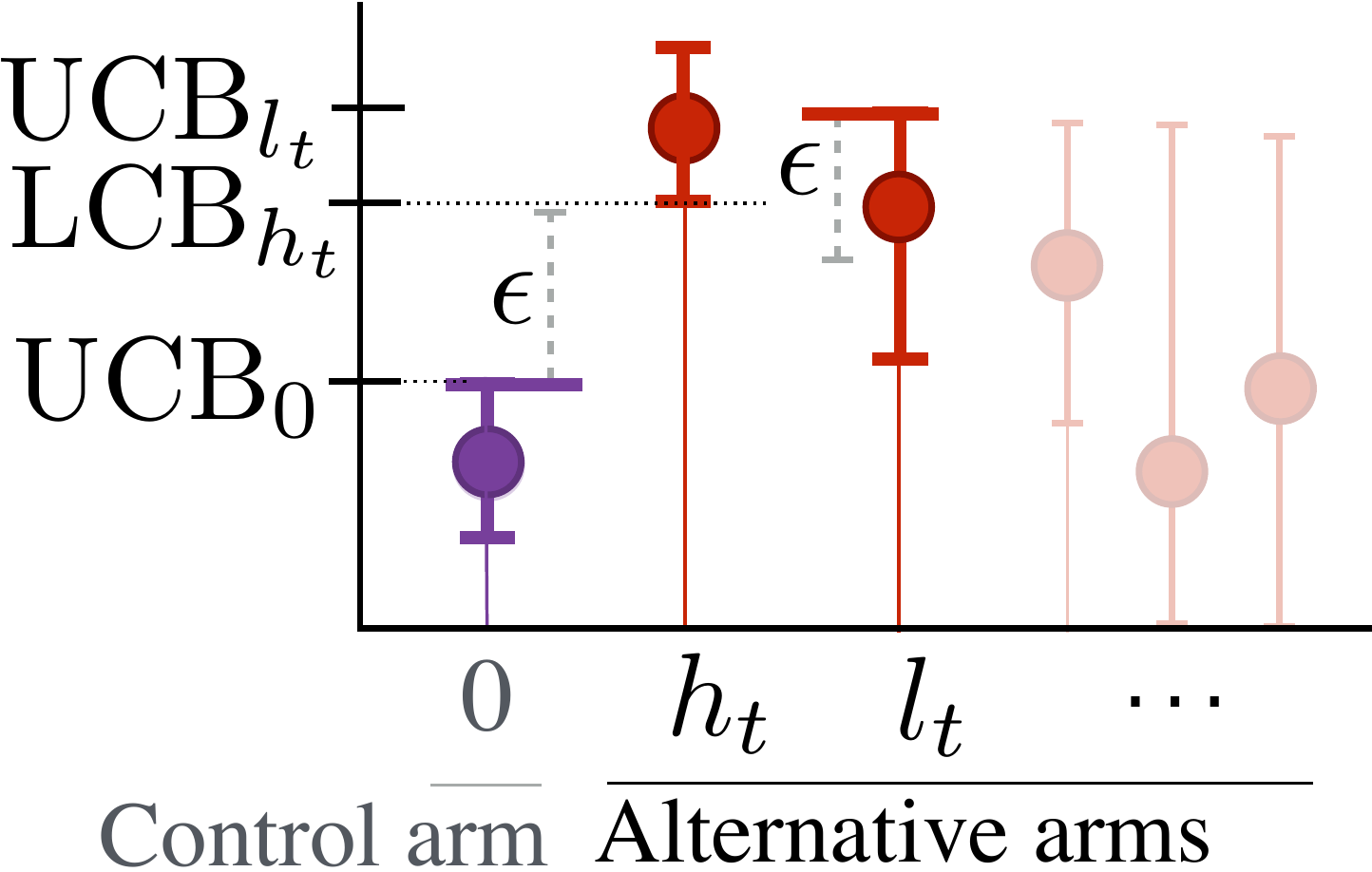} \\
  (a) & (b)
  \end{tabular}
  \caption{(a) The means of arms $\{1,2,3\}$ are within $\epsilon$ of
    the best arm, but only arms $\{1,2\}$ are at least $\epsilon$
    better than the control arm 0. Thus, returning any of arms
    $\{3,4,5\}$ would result in a false discovery when $\epsilon
    >0$. (b) An example of the stopping condition being critically met
    and returning a non-control arm $h_t$. While $\LCB_{h_t} >
    \UCB_{\ell_t} - \epsilon$ is satisfied with some slack, $\LCB_{h_t} >
    \UCB_{0} + \epsilon$ is just barely satisfied.}
  \label{fig:LUCBEps}
\end{figure}

The following proposition applies to Algorithm~\ref{AlgoModLUCB} run
with a control arm indexed by $i = 0$ with mean $\mu_0$ and
alternative arms indexed by $i=1,\dots,K$ with means $\mu_i$,
respectively.  Let $\Output$ denote the random arm returned by the
algorithm assuming that it exits, and define the set
\vspace{-0.03in}
\begin{align}
 \label{EqnDefnSetStar}
  \SetStar & \defn \{ \realbest \neq 0 \mid \mu_{\realbest} \geq \max
  \limits_{i=1,\dots,K} \mu_i - \epsilon \quad \mbox{and} \quad
  \mu_{\realbest} > \mu_0 + \epsilon \}.
\end{align}
Note that the mean associated with any index $\realbest \in \SetStar$,
assuming that the set is non-empty, is guaranteed to be
$\epsilon$-superior to the control mean, and at most
$\epsilon$-inferior to the maximum mean over all arms.

\begin{props}
  \label{PropLUCBEps}
  The algorithm~\ref{AlgoModLUCB} terminates in finite time with probability one.
Furthermore, suppose that the samples from each arm are independent and
sub-Gaussian with scale $1$.  Then for any $\delta \in (0,1)$ and
$\epsilon \geq 0$, Algorithm~\ref{AlgoModLUCB} has the following
guarantees:
\begin{enumerate}[(a)]
\item Suppose that $\mu_0 > \max \limits_{i=1,\dots,K} \mu_i -
  \epsilon$.  Then with probability at least $1-\delta$, the algorithm exits with $\Output = 0$ after taking at most $O\left(\sum_{i=0}^K \effgap_i^{-2} \log(K \log(\effgap_i^{-2})/\delta) \right)$ time steps with effective gaps 
  \begin{align*}
  \effgap_0 &= (\mu_0 + \epsilon) - \max \limits_{j=1,\dots,K} \mu_j  \: \text{ and} \\
  \effgap_i &=  (\mu_0 + \epsilon) -  \mu_i .
  \end{align*}
\item Otherwise, suppose that the 
  set $\SetStar$ as defined in equation~\eqref{EqnDefnSetStar} is non-empty.  Then with probability  at least $1-\delta$, the algorithm exits with $\Output \in \SetStar$ after taking at most \\$O\left(\sum_{i=0}^K \effgap_i^{-2} \log(K \log(\effgap_i^{-2})/\delta) \right)$ time steps with effective gaps
  \begin{align*}
  \effgap_0   &= \min\left\{ \max \limits_{j=1,\dots,K} \mu_j - (\mu_0+\epsilon), \max\{ \Delta_0, \epsilon \} \right\} \: \text{ and}\\
  \effgap_i &= \max\left\{ \Delta_i ,  \min\left\{ \max \limits_{j=1,\dots,K} \mu_j - (\mu_0+\epsilon), \epsilon \right\}  \right\}.
  \end{align*}
\end{enumerate}
\end{props}

\noindent See Section~\ref{SecProofPropLUCBEps} for the proof of this
claim.  Part (a) of Proposition~\ref{PropLUCBEps} guarantees that when
no alternative arm is $\epsilon$-superior to the control arm
(i.e. under the null hypothesis), the algorithm stops and returns the
control arm after a certain number of samples with probability at
least $1 - \delta$, where the sample complexity depends on
$\epsilon$-modified gaps between the means $\mu_0$ and $\mu_i$.  Part
(b) guarantees that if there is in fact at least one alternative that
is $\epsilon$-superior to the control arm (i.e. under the
alternative), then the algorithm will find at least one of them that
is at most $\epsilon$-inferior to the best of all possible arms with
the same sample complexity and probability.

Note that the required number of samples
$O\left(\sum_{i=0}^K \effgap_i^{-2} \log(K
  \log(\effgap_i^{-2})/\delta) \right)$ in
Proposition~\ref{PropLUCBEps} is comparable, up to log factors, with
the well-known results in \cite{KTAS12,SJR17} for the case
$\epsilon = 0$, with the modified gaps $\effgap_i$ replacing $\Delta_i = \mu_{\realbest}- \mu_i$.
Indeed, the
nearly optimal sample complexity result of \cite{SJR17}
implies that
the algorithm terminates under settings (a) and (b) after at most
$O(\max_{j \neq i_\star} \Delta_{j}^{-2} \log( K
\log(\Delta_{j}^{-2})/\delta) + \sum_{i \neq i_\star} \Delta_{i}^{-2}
\log( \log(\Delta_{i}^{-2})/\delta)$) samples are taken.

In our development to follow, we now bring back the index for
experiment $j$, in particular using $P^j$ to denote the quantity
$P^j_T$ at any stopping time $T$.  Here the stopping time can either
be defined by the scientist, or in an algorithmic manner.


\subsection{Best-arm MAB interacting with online FDR}

After having established null hypotheses and $p$-values in the context
of best-arm MAB algorithms, we are now ready to embed them into an
online FDR procedure.  In the following, we consider $p$-values for
the $j$-th experiment $\nullp{j} \defn \nullp{j}_{\rowstop_j}$ which
is just the $p$-value as defined in equation~\eqref{EqnPVal} at the
stopping time $\rowstop_j$, which depends on $\alpha_j$.

We denote the set of true null and false null hypotheses up to
experiment $J$ as $\truenulls(J)$ and $\falsenulls(J)$ respectively,
where we drop the argument whenever it's clear from the context.
The variable $R_j = \Indi_{P^j \leq \alpha_j}$
indicates whether a the null hypothesis of experiment
$j$ has been rejected, where $R_j=1$ denotes a claimed discovery
that an alternative was better than the control.  
The false discovery
rate (FDR) and modified FDR \emph{up to experiment $\numexp$} are then
defined as
\vspace{-0.03in}
\begin{align}
\FDR(\numexp) & \defn \EE \frac{\sum_{j\in \truenulls}
  \rej_j}{\sum_{i=1}^\numexp \rej_i \vee 1} \qquad \text{ and } \qquad
\mFDR(\numexp) \defn \frac{\EE \sum_{j\in \truenulls} \rej_j}{\EE
  \sum_{i=1}^\numexp \rej_i + 1}.
\end{align}
Here the expectations are taken with respect to distributions of the
arm pulls and the respective sampling algorithm.  In general, it is
not true that control of one quantity implies control of the other.
Nevertheless, in the long run (when the law of large numbers is a good
approximation), one does not expect a major difference between the two
quantities in practice.

The set of true nulls $\truenulls$ thus includes all experiments where
$\thenull^j$ is true, and the $\FDR$ and $\mFDR$ are well-defined for any
number of experiments $\numexp$, since we often desire to
control $\FDR(J)$ or $\mFDR(J)$ for all $J \in \NN$.  In order to
measure power, we define the \emph{$\eps$-best-arm discovery rate} as
\begin{align}
\label{EqBDRDef}
 \epsBDR (\numexp) & \defn \frac{\EE \sum_{j\in \falsenulls} \rej_j
   \Indi_{\mu_{\Output} \geq \mumax -\eps }\Indi_{\mu_{\Output}
     \geq \mu_0 + \eps}}{|\falsenulls(\numexp)|}
\end{align}

\vspace{-0.05in} We provide a concrete procedure~\ref{ProcLORD} for
our doubly sequential framework, where we use a particular online FDR
algorithm due to Javanmard and Montanari~\cite{JM16} known as LORD;
the reader should note that other online FDR procedure could be used
to obtain essentially the same set of guarantees.  Given a desired
level $\alpha$, the LORD procedure starts off with an initial
``$\alpha$-wealth'' of $\wlth(0) < \alpha$.  Based on a inifinite
sequence $\{\gamma_i\}_{i=1}^\infty$ that sums to one, and the time of
the most recent discovery $\tau_j$, it uses up a fraction
$\gamma_{j-\tau_j}$ of the remaining $\alpha$-wealth to test.
Whenever there is a rejection, we increase the $\alpha$-wealth by
$\alpha - \wlth(0)$.  A feasible choice for a stopping time in
practice is
\mbox{$\rowstop_j \defn \min\{ \rowstop(\alpha_j), \trunctime\}$,}
where $\trunctime$ is a maximal number of samples the scientist wants
to pull and $\rowstop(\alpha_j)$ is the stopping time of the best-arm
MAB algorithm run at confidence $\alpha_j$.


\begin{procedure}[htbp]
\caption{MAB-LORD: best-arm identification with online FDR control}
\label{ProcLORD}
\begin{enumerate}
\item Initialize $\wlth(0) < \alpha$, set $\tau_0=0$, and choose a
  sequence $\{\gamma_i\}$ s.t. $\sum_{i=1}^\infty \gamma_i = 1$
\item At each step $j$, compute $\alpha_j = \gamma_{j-\tau_j}
  \wealth(\tau_j)$ and \\ $\wealth(j+1) = \wealth(j) - \alpha_j +
  \rej_j (\alpha - \wlth(0))$

\item Output $\alpha_j$ and run Algorithm~\ref{AlgoModLUCB} using
  $\putinalpha$-confidence and stop at a stopping time $\rowstop_j$.
\item Algorithm~\ref{AlgoModLUCB} returns $\nullp{j}$ and we reject the null
  hypothesis if $\nullp{j} \leq \alpha_j$.
\item Set $\rej_j = \Indi_{\nullp{j}\leq \alpha_j}, \tau_j =
  \tau_{j-1} \vee jR_j$, update $j = j+1$ and go back to step 2.
\end{enumerate}
\end{procedure}

\noindent The following theorem provides guarantees on $\mFDR$ and power 
for the MAB-LORD procedure.

\begin{theos}[Online mFDR control for MAB-LORD]
\
\label{ThmOnlinemFDR}
\vspace{-0.07in}
\begin{enumerate}[(a)]
\item Procedure~\ref{ProcLORD} achieves mFDR control at level $\alpha$
  for stopping times \mbox{$\rowstop_j =\min\{ \rowstop(\putinalpha),
    \trunctime\}$.}
\item Furthermore, if we set $\trunctime = \infty$,
  Procedure~\ref{ProcLORD} satisfies
\vspace{-0.05in}
\begin{align}
\epsBDR(\numexp) \geq \frac{\sum_{j=1}^\numexp \Indi_{j\in
    \falsenulls} (1- \alpha_{j})}{|\falsenulls(\numexp)|}.
\end{align}
\end{enumerate}
\end{theos}

The proof of this theorem can be found in Section~\ref{SecProofThm1}.
Note that by the arguments in the proof of
Theorem~\ref{ThmOnlinemFDR}, mFDR control itself is actually
guaranteed for any generalized $\alpha$-investing
procedure~\cite{AR14} combined with any best-arm MAB algorithm. In
fact we could use any adaptive stopping time $T_j$ which depend on the
history only via the rejections $R_1,\dots, R_{j-1}$. Furthermore,
using a modified LORD proposed by Javanmard and Montanari~\cite{JM15},
we can also guarantee FDR control-- which can be found in
Appendix~\ref{SecFDRControl}.

It is noteworthy that small values of $\alpha$ do not only guarantee
smaller $\FDR$ error but also higher $\BDR$.  However, there is no
free lunch --- a smaller $\alpha$ implies a smaller $\alpha_j$ at each
experiment, which in turn causes the best-arm MAB algorithm to employ a
larger number of pulls in each experiment.

\vspace{-0.2in}
\section{Experimental results}
\label{SecExperiments}

In the following, we describe the results of experiments \footnote{The
  code for reproducing all experiments and plots in this paper is
  publicly available at \texttt{https://github.com/fanny-yang/MABFDR}}
on both simulated and real-world data sets to illustrate the
properties and guarantees of our procedure described in
Section~\ref{SecConcrete}.  In particular, we show that the mFDR is
indeed controlled over time and that MAB-FDR (used interchangeably
with MAB-LORD here) is highly advantageous in terms of sample
complexity and power compared to a straightforward extension of A/B
testing that is embedded in online FDR procedures. Unless otherwise
noted, we set $\epsilon =0$ in all of our simulations to focus on the
main ideas and keep the discussion concise.

There are two natural frameworks to compare against MAB-FDR. 
The first, called \fABFDR{} or AB-LORD, swaps the MAB part for an A/B (i.e. A/B/n) test
 (uniformly sampling all alternatives until termination).
 The second comparator swaps
the online FDR control for independent testing at $\alpha$
for all hypotheses  -- we call this \fMABIND{}.
Formally, \fABFDR{} swaps step 3 in Procedure~\ref{ProcLORD} with
``\emph{Output $\alpha_j$ and uniformly sample each arm until stopping
  time $\rowstop_j$.}'' while \fMABIND{} swaps step 4 in
Procedure~\ref{ProcLORD} with ``\emph{The algorithm returns
  $\nullp{j}$ and we reject the null hypothesis if $\nullp{j} \leq
  \alpha$.}''.  In order to compare the performances of these
procedures, we ran three sets of simulations using
Procedure~\ref{ProcLORD} with $\epsilon = 0$ and $\gamma_j = 0.07
\:\frac{\log (j \vee 2)}{j \E^{\sqrt{\log j}}}$ as in~\cite{JM16}.
The first two sets are on artificial data (Gaussian and Bernoulli
draws from sets of randomly drawn means $\mu_i$), while the third is
based on data from the New Yorker Cartoon Caption Contest (Bernoulli
draws).

Our experiments are run on artificial data with Gaussian/Bernoulli
draws and real-world Bernoulli draws from the New Yorker
Cartoon Caption Contest.
Recall that the sample complexity of the best-arm MAB algorithm is
determined by the gaps $\Delta_j = \mu_{\realbest} - \mu_j$. One of
the main relevant differences to consider between an experiment of
artificial or real-world nature is thus the distribution of the means
$\mu_i$ for $i=1,\dots,\numarms$.  The artificial data simulations are
run with a fixed gap between the mean of the best arm
$\mu_{\realbest}$ and second best arm $\mu_2$, which we denote by
$\Delta = \mu_{\realbest} - \mu_2$. In each experiment (hypothesis),
the means of the other arms are set uniformly in $[0, \mu_2]$.  For
our real-world simulations with the cartoon contest, the means for the
arms in each experiment are not arbitrary but correspond to empirical
means from the caption contest. In addition, the contests actually
follow a natural chronological order (see details below), which makes
this dataset highly relevant to our purposes. In all simulations, 60\%
of all the hypotheses are true nulls, and their indices are chosen
uniformly.



\subsection{Power and sample complexity}
\label{SecPower}
The first set of simulations compares \fMABFDR{} against
\fABFDR{}. They confirm that the total number of necessary pulls to
determine significance (which we refer to as \emph{sample complexity})
is much smaller for \fMABFDR{} than for \fABFDR{}.
In the \fMABFDR{} framework, this also effectively leads to higher
power given a fixed truncation time. 

Two types of plots are used to demonstrate the superiority of our
procedure: for one we fix the number of arms and plot the $\epsBDR$
with $\epsilon = 0$ (which we call $\BDR$ for short) for both
procedures over different choices of truncation times $\trunctime$.
For the other we fix $\trunctime$ and show how the sample complexity
varies with the number of arms.  Note that low $\BDR$ means that the
bandit algorithm often reaches truncation time before it could stop.


\vspace{-0.1in}
\subsubsection{Simulated Gaussian and Bernoulli trials}

For the Gaussian draws, we set $\mu_{\realbest} = 8$.  The gap to the
second best arm is $\Delta = 3$ so that all means $\mu_{i \neq
  \realbest}$ are drawn uniformly between $Unif \sim [0,5]$. The
number of hypotheses is fixed to be $500$.  For Bernoulli draws we
choose the maximum mean to be $\mu_{\realbest} = 0.4$, $\Delta = 0.3$
so that all means $\mu_{i\neq \realbest}$ are drawn uniformly between
$Unif \sim [0,0.1]$. The number of hypotheses is fixed at $50$.  We
display the empirical average over $100$ runs where each run uses the
same hypothesis sequence (indicating which hypotheses are true and
false) and sequence of means $\mu_i$ for each hypothesis. The only
randomness we average over comes from the random Gaussian/Bernoulli
draws which cause different rejections $R_j$ and $\alpha_j$, so that
the randomness in each draw propagates through the online FDR
procedure.  The results can be seen in Figures~\ref{FigGaussian}
and~\ref{FigBer}.

\begin{figure}[htbp]
\begin{center}
\begin{tabular}{cc}
\widgraph{.45\textwidth}{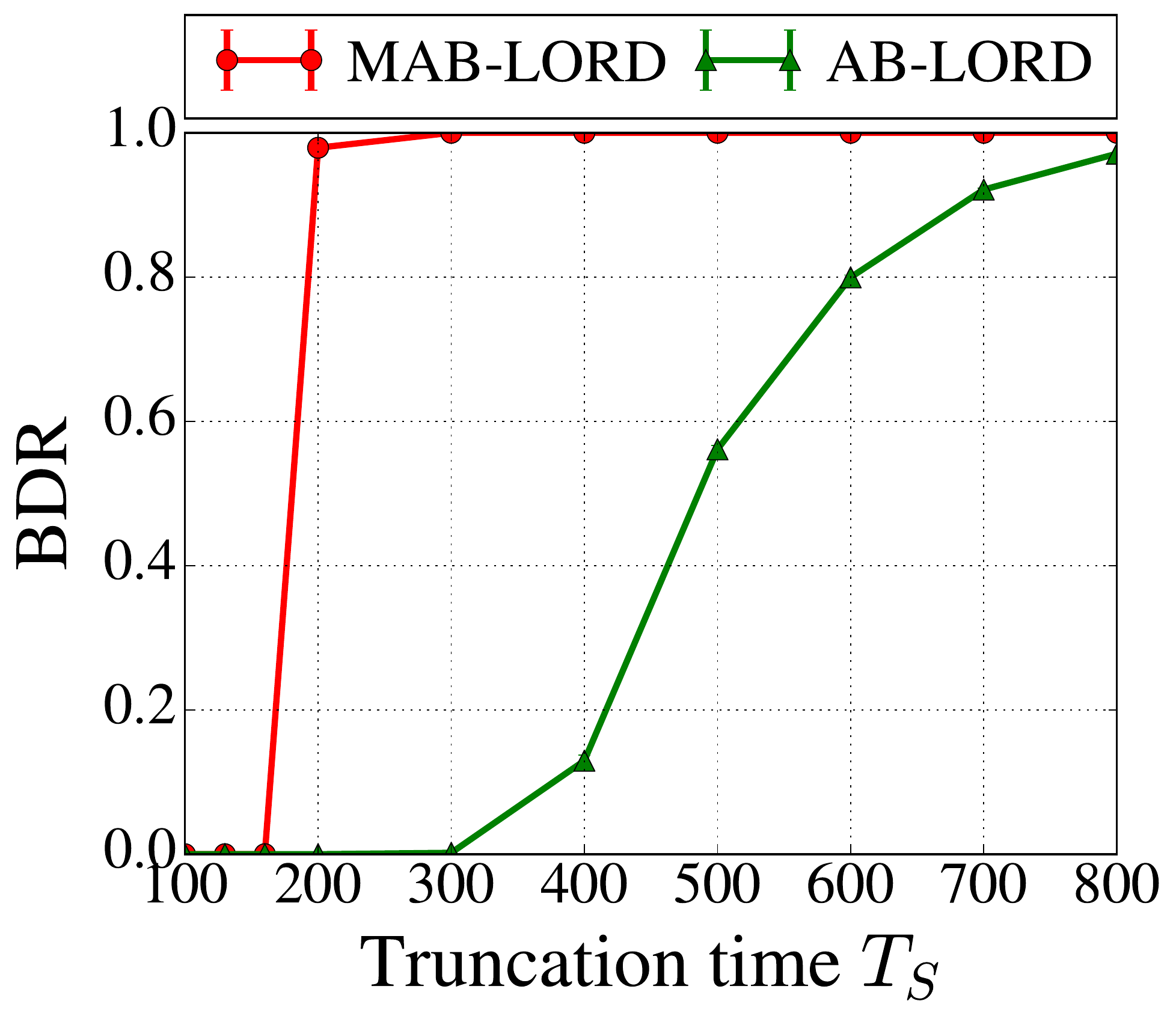} &
\widgraph{.45\textwidth}{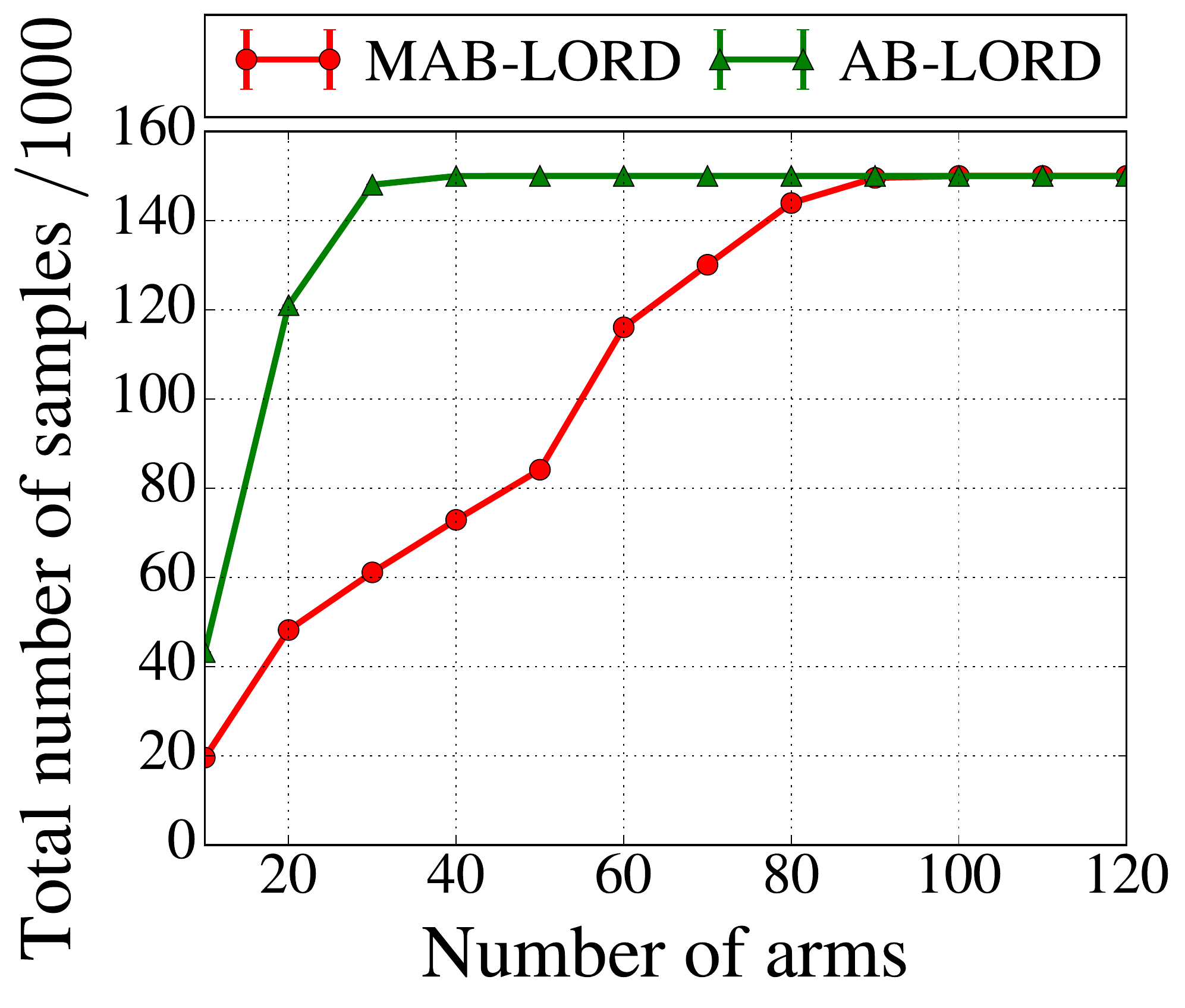} \\
(a) & (b)
\end{tabular}
\vspace{-0.1in}
\end{center}
\caption{(a) Power vs. truncation time $T_S$ (per
  hypothesis) for $50$ arms and (b) Sample complexity vs. \#
  arms for truncation time $\trunctime = 300$ for Gaussian draws with
  fixed $\mu_{\realbest} = 8$, $\Delta = 3$ over $500$ hypotheses with $200$
  non-nulls, averaged over $100$ runs.}
\label{FigGaussian}
\vskip -0.1in
\end{figure} 

The power at any given truncation time is much higher for \fMABFDR{}
than \fABFDR{}.  This is because the best-arm MAB is more likely to
satisfy the stopping criterion before any given truncation time than
the uniform sampling algorithm.
The plot in Fig.~\ref{FigGaussian}(a) suggests
that the actual stopping time of the algorithm is concentrated between
$160$ and $200$ while it is much
more spread out for the uniform algorithm.

\begin{figure}[h!]
\vskip 0.2in
\begin{center}
\begin{tabular}{cc}
\widgraph{.45\textwidth}{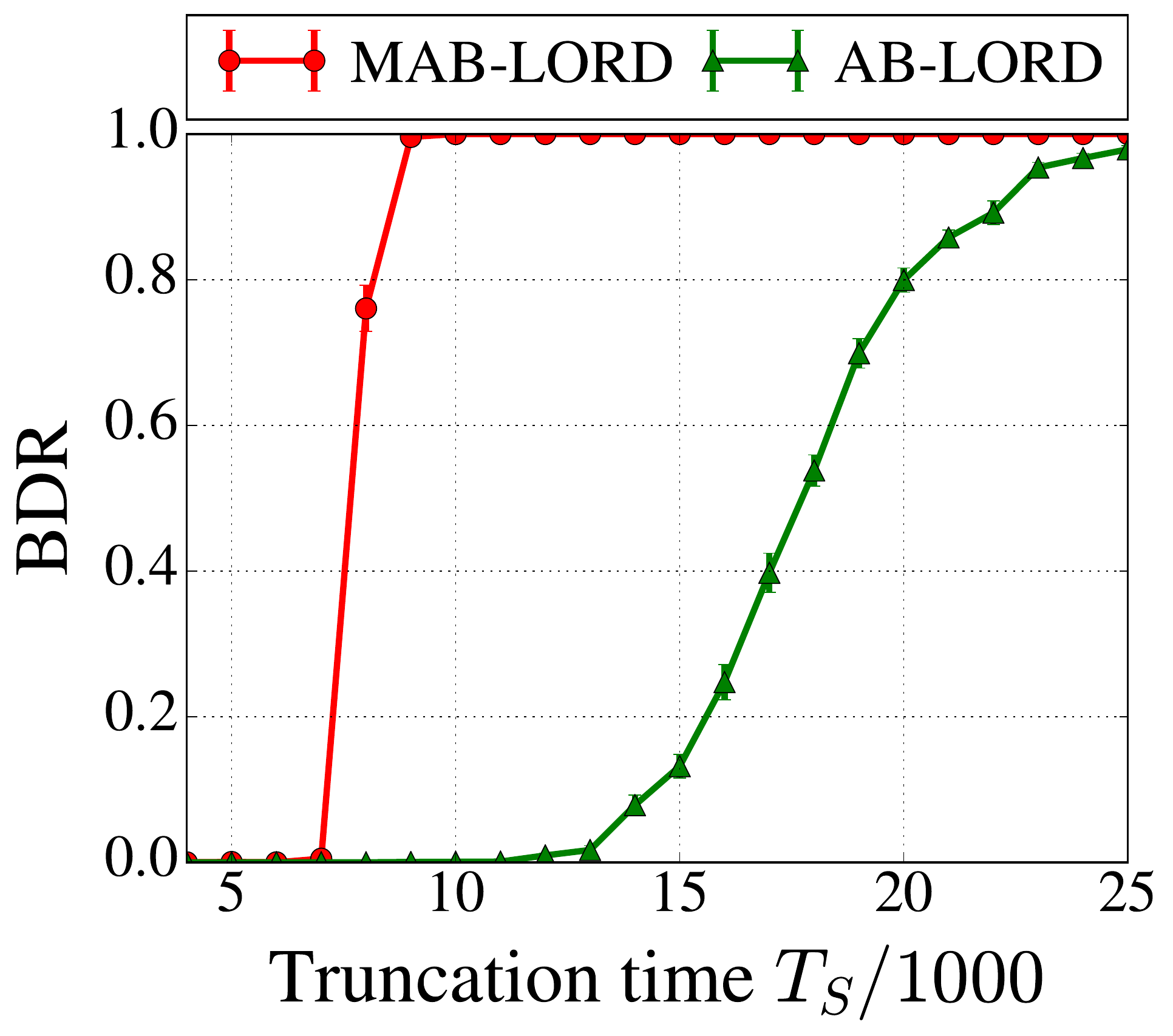} & 
\widgraph{.45\textwidth}{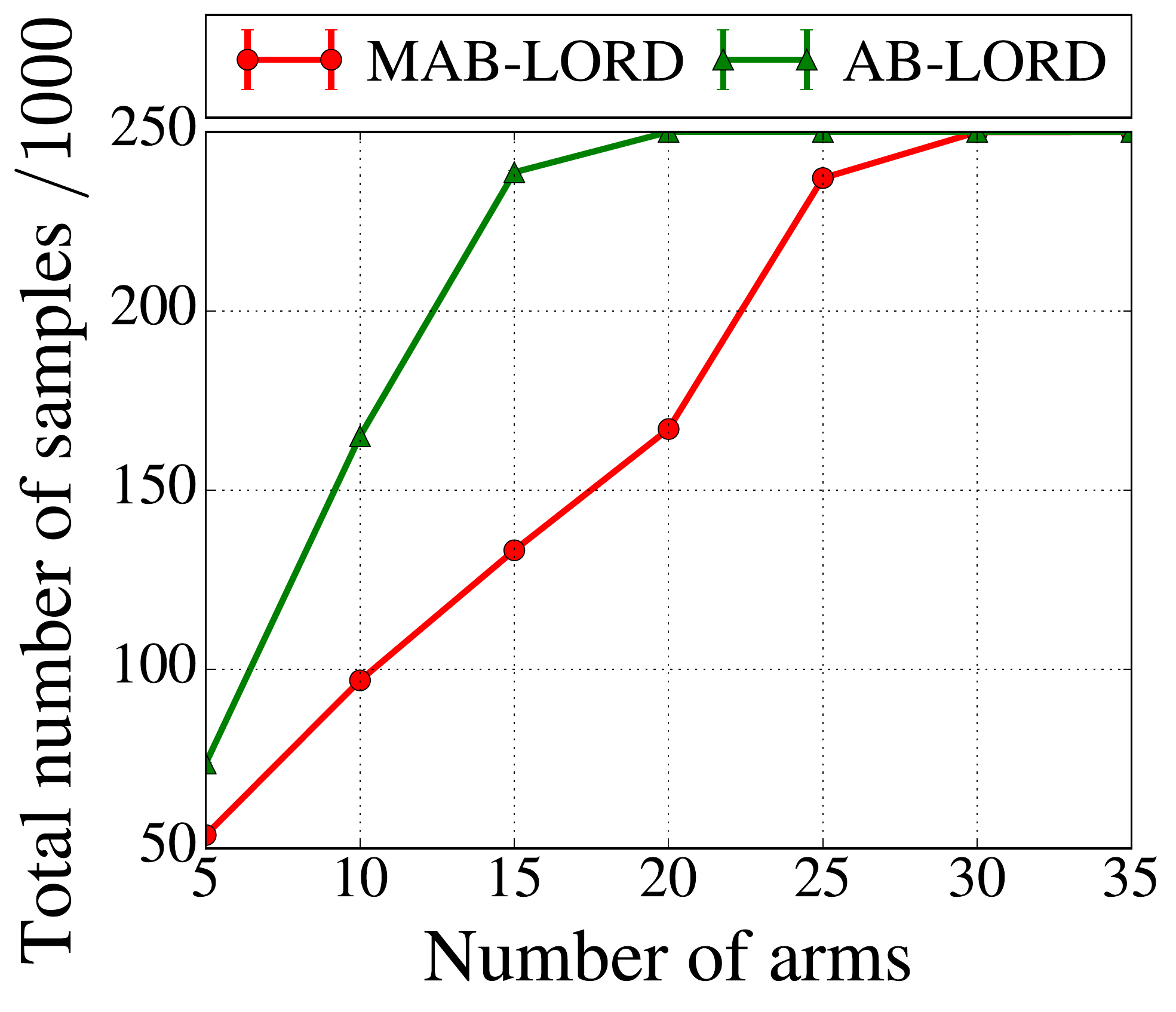} \\
(a) & (b)
\end{tabular}
\vspace{-0.1in}
\end{center}
\caption{(a) Power over truncation time $T_S$ (per
  hypothesis) for $50$ arms and (b) Sample complexity over number of
  arms for truncation time $\trunctime = 5000$ for Bernoulli draws with
  fixed $\mu_{\realbest} = 0.7$, $\Delta = 0.3$ over $50$ hypotheses with $20$
  non-nulls, averaged over $100$ runs.}
\label{FigBer}
\vskip -0.2in
\end{figure}

The sample complexity plot in Fig.~\ref{FigGaussian}(b) qualitatively
shows how the total number of necessary arm pulls for \fABFDR{}
increases much faster with the number of arms than for the \fMABFDR{},
before it plateaus at the truncation time multiplied by the number of
hypotheses. Recall that whenever the best-arm MAB stops before the
truncation time in each hypothesis, the stopping criterion is met,
i.e. the best arm is identified with probability at least
$1-\alpha_j$, so that the power is bound to be close to one whenever
$\rowstop_j = T(\alpha_j)$.



For Bernoulli draws we choose the maximum mean to be $\mu_{\realbest}
= 0.4$, $\Delta = 0.3$ so that all means $\mu_{i\neq \realbest}$ are
drawn uniformly between $Unif \sim [0,0.1]$. The number of hypotheses
is fixed at $50$.  Otherwise the experimental setup is identical to
those discussed in the main text for Gaussians.  The plots for
Bernoulli data can be found in Fig.~\ref{FigBer}.

The behavior for both Gaussian and Bernoullis are comparable, which is
not surprising due to the choice of the subGaussian LIL bound.
However one may notice that the choice of the gap of $\Delta = 3$
vs. $\Delta = 0.3$ drastically increases sample complexity so that the
phase transition for power is shifted to very large $T_S$.

\vspace{-0.1in}
\subsubsection{Application to New Yorker captions}

In the simulations with real data we consider the crowd-sourced data
collected for the \emph{New Yorker Magazine's} Cartoon Caption
contest: for a fixed cartoon, captions are shown to individuals online
one at a time and they are asked to rate them as `unfunny', `somewhat
funny', or `funny'.  We considered 30 contests\footnote{Contest
  numbers 520-551, excluding 525 and 540 as they were not
  present. Full dataset and its description is available at
  \url{https://github.com/nextml/NEXT-data/}.} where for each contest,
we computed the fraction of times each caption was rated as either
`somewhat funny' or `funny'.  We treat each caption as an arm, but
because each caption was only shown a finite number of times in the
dataset, we simulate draws from a Bernoulli distribution with the
observed empirical mean computed from the dataset.  When considering
subsets of the arms in any given experiment, we always use the
captions with the highest empirical means (i.e. if $n=10$ then we use
the $10$ captions that had the highest empirical means in that
contest).

Although \fMABFDR{}\: still outperforms \fABFDR{}\: by a large margin,
the plots in Figure~\ref{FigNY} also show how the power and sample
complexity notably differ from our toy simulation, where we seem to
have chosen a rather benign distribution of means - in this setting,
the gap $\Delta$ is much lower, often around $\sim 0.01$.

\begin{figure}[h!]
\begin{center}
\begin{tabular}{cc}
 \widgraph{.45\textwidth}{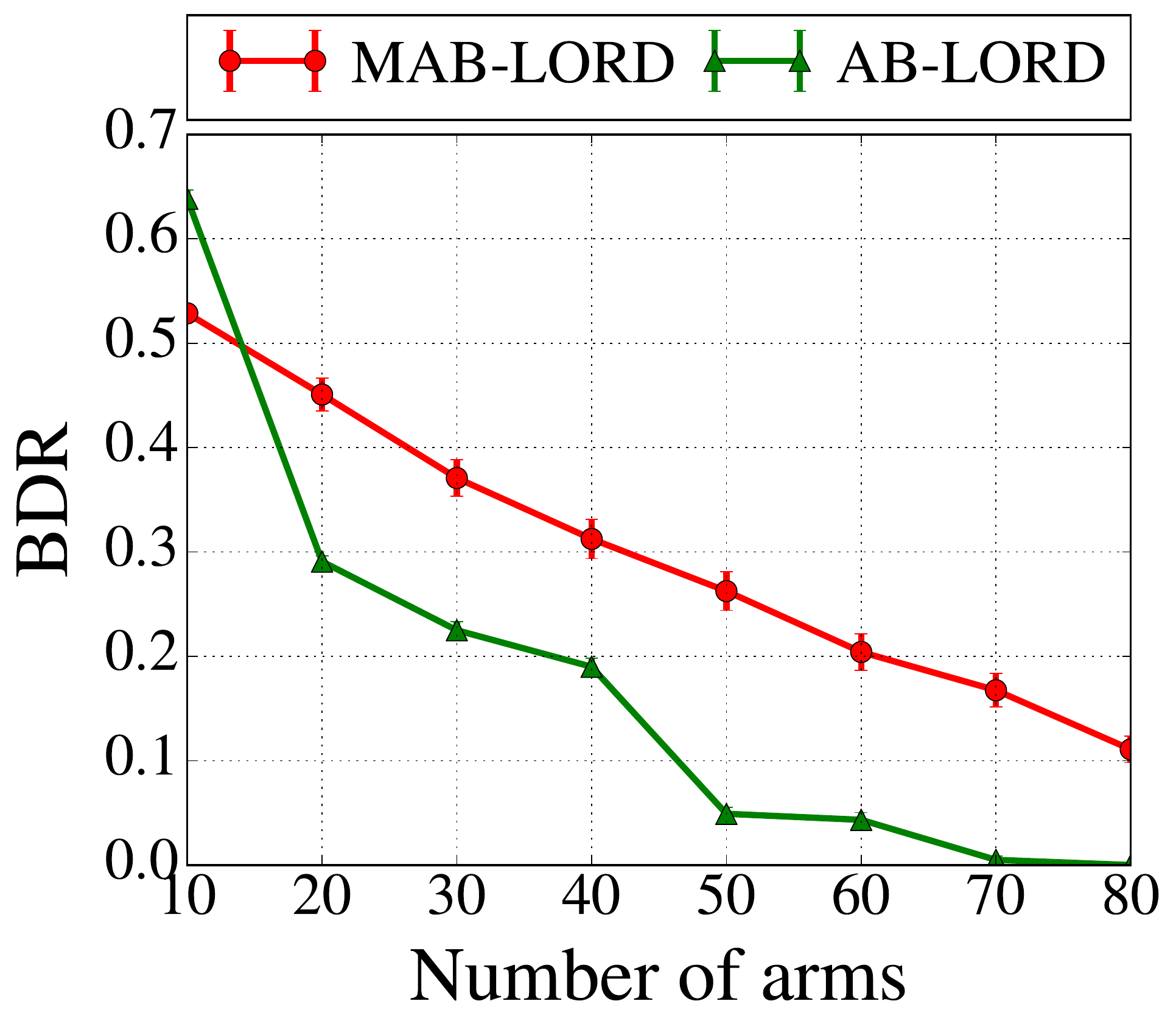} & \hspace{.5cm}
\widgraph{.45\textwidth}{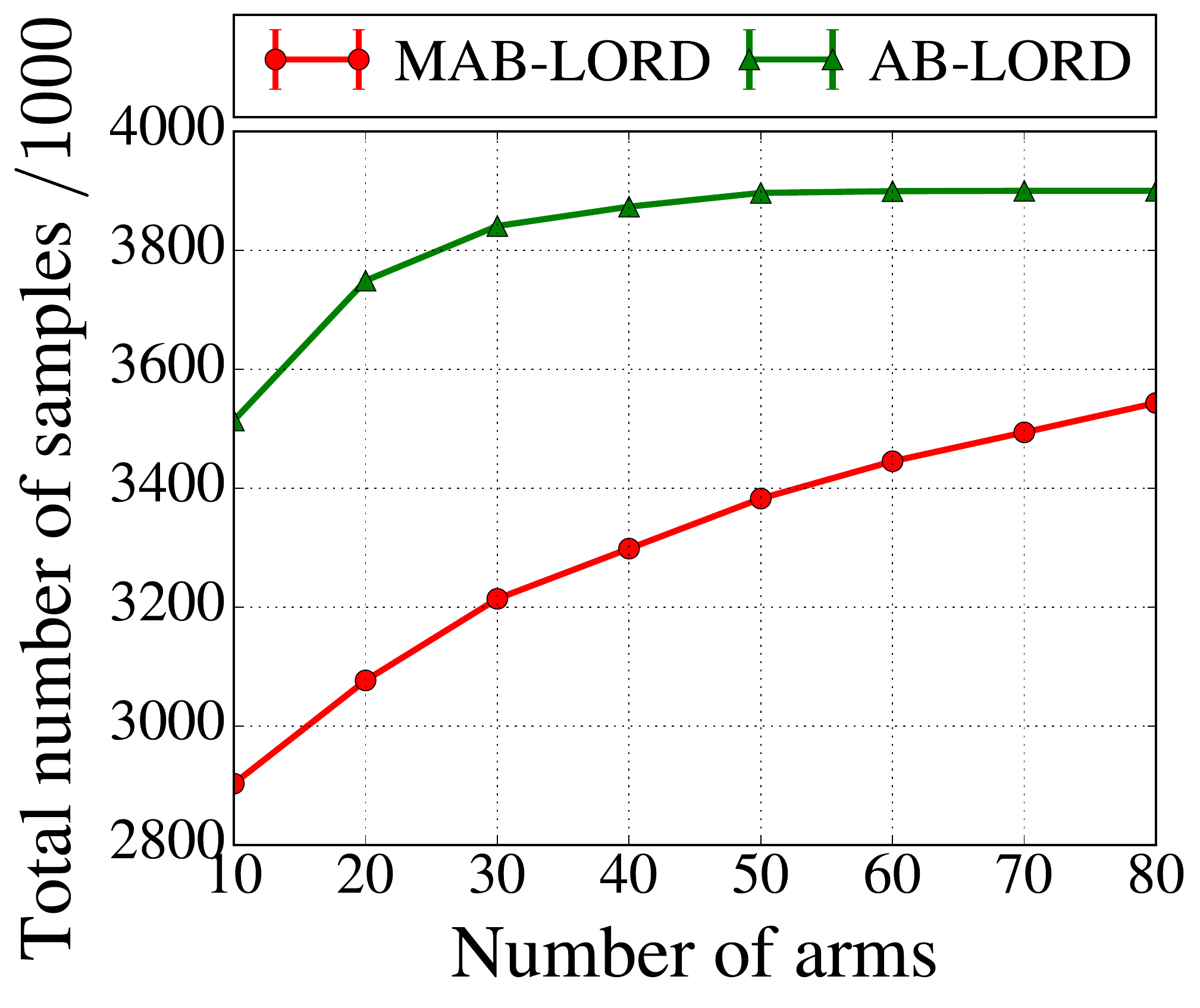} \\
(a) & (b)
\end{tabular}
\vspace{-0.1in}
\end{center}
\caption{(a) $\BDR$ over number of arms, i.e. truncation time per
  hypothesis for $10$ arms and (b) Sample complexity over number of
  arms for truncation time $\trunctime = 130000$ for Bernoulli draws,
  $30$ hypotheses with $12$ non-nulls and averaged over $100$
  runs. }
\label{FigNY}
\vskip -0.1in
\end{figure}

\vspace{-0.1in}
\subsection{mFDR and FDR control}

In this section we use simulations to demonstrate the second part of
our meta algorithm which deals with the control of the false discovery
rate or its modified version.  Since bandit algorithms have a very
high best-arm discovery guarantee which in practice even exceeds its
theoretical guarantee of at least $1-\alpha_j$, mFDR and FDR plots on
\fMABFDR{}\: directly do not lead to very insightful plots - namely
the constant $0$ line.  However, we can demonstrate that even under
adversarial conditions, i.e. when the $\nullp{}$-value under the null
is much less concentrated around one than obtained via the best arm
bandit algorithm, mFDR or the false discovery proportion (FDP) in each
run are still controlled \emph{at any time $t$} as
Theorem~\ref{ThmOnlinemFDR} guarantees. Albeit not exactly reflecting
mFDR control in the case of \fMABFDR{}\: but in fact in an even harder
setting, results from these experiments can be regarded as valuable on
their own - it emphasizes the fact that Theorem~\ref{ThmOnlinemFDR}
guarantees mFDR control independent of the adaptive sampling algorithm
and specific choice of $p$-value as long as it is always valid.

For Figure~\ref{FigFDR}, we again consider Gaussian draws with the
same settings as described in~\ref{SecPower}. This time however, for
each true null hypothesis we skip the bandit experiment and directly
draw $\nullp{j} \sim [0,1]$ to compare with the significance levels
$\alpha_j$ from our online FDR procedure~\ref{ProcLORD}. As mentioned
above, by Theorem~\ref{ThmOnlinemFDR}, mFDR should still be controlled
as it only requires the $p$-values to be super-uniform.  In
Figure~\ref{FigFDR}(a) we plot the instantaneous false discovery
proportion (number of false discoveries over total discoveries)
$\FDP(J) = \frac{\sum_{j\in \truenulls{J}} R_j}{\sum_{j = 1}^T R_j} $
over the hypothesis index for different runs with the same settings.
Apart from fluctuations in the beginning due to the relatively small
denominator, we can observe how the guarantee for the $\FDR(J) = \EE
\:\:\FDP(J)$, with its empirical value  depicted by the red line,
transfers to the control of each individual run (blue lines).

In Figure~\ref{FigFDR}, we compare the mFDR (which in fact coincides
with the FDR in this plot) of MAB-FDR using different multiple testing
procedures, including MAB-IND and a Bonferroni type correction.  The
latter uses a simple union bound and chooses $\alpha_j$ such that
$\sum_{j =1}^\infty \alpha_j \leq \alpha$ and thus trivially allows
for any time FWER, and thus FDR control. In our simulations we use
$\alpha_j = \frac{6 \alpha}{\pi^2 j^2}$.  As expected, Bonferroni is
too conservative and barely makes any rejections whereas the naive
MAB-IND approach does not control FDR. LORD avoids both extremes and
controls FDR while having reasonable power.

\begin{figure}[t!]
\vskip -0.2in
\begin{center}
\hspace*{-0.1in}
\begin{tabular}{cc}
\hspace{-.5cm}\widgraph{.55\textwidth}{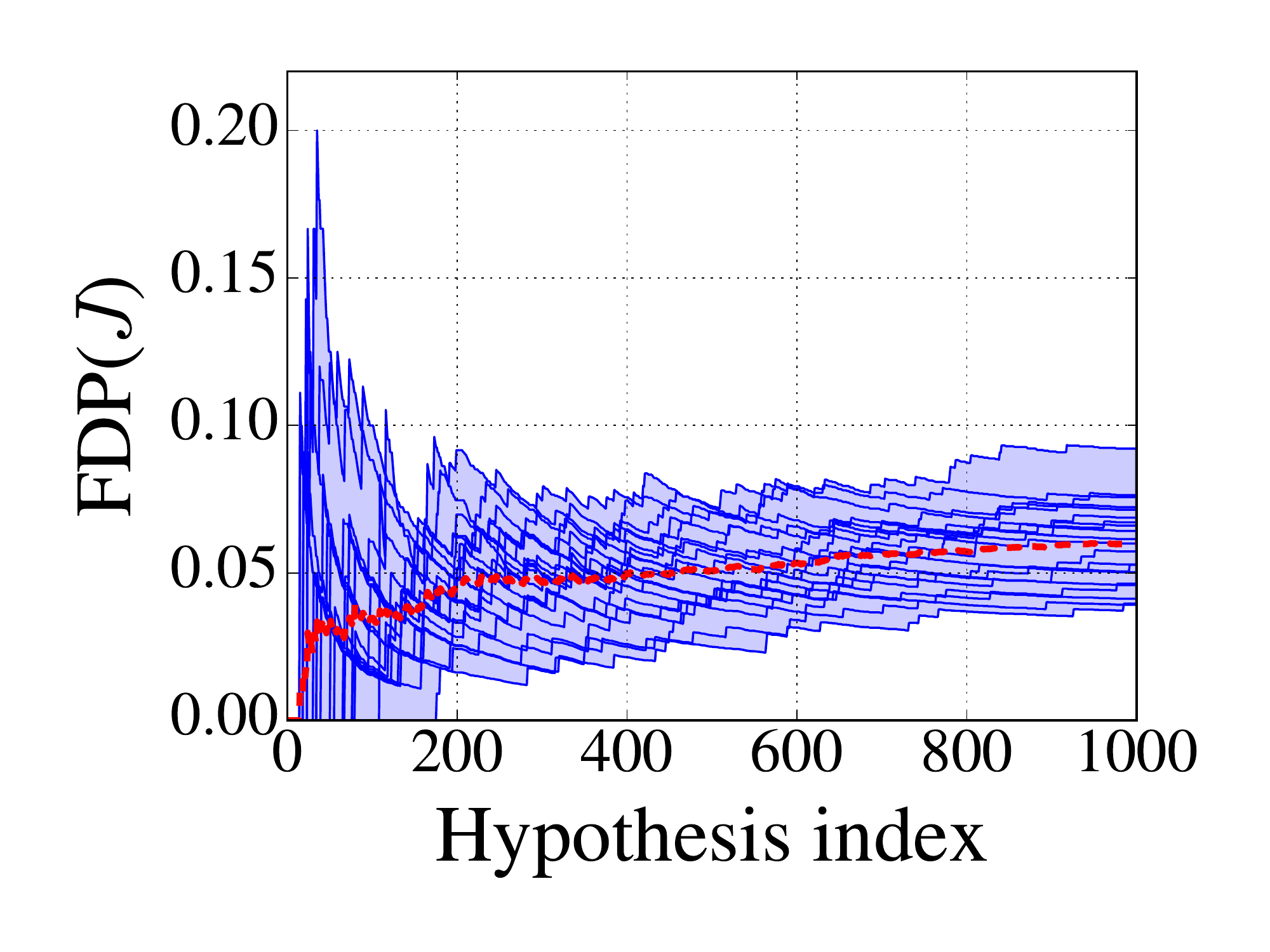} &  \hspace{-.5cm}
\widgraph{.50\textwidth}{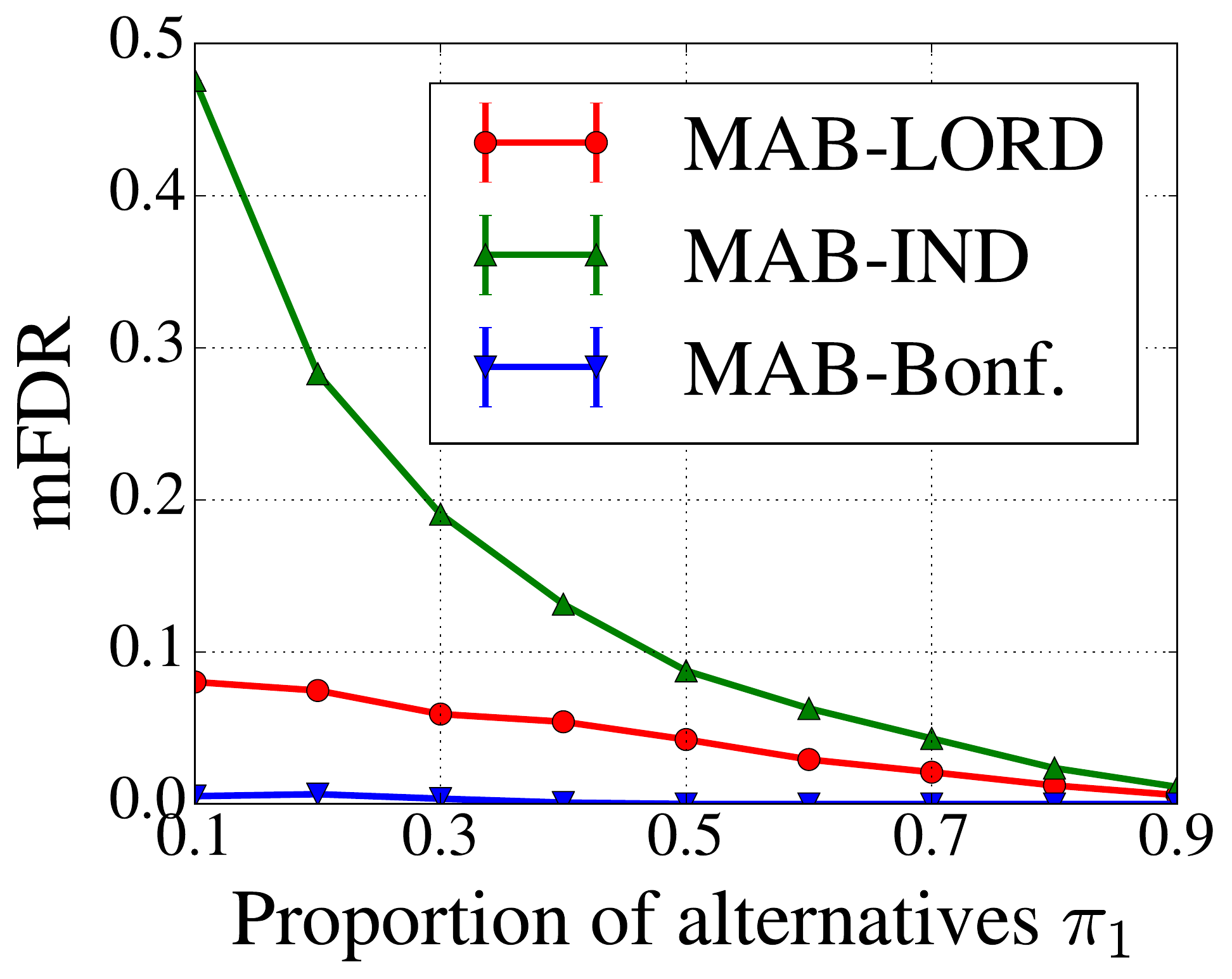} \\
(a) & (b)
\end{tabular}
\end{center}
\caption{(a) Single runs of MAB-LORD (blue) and their average (red)
  with uniformly drawn $p$-values for null hypotheses and Gaussian
  draws for non-nulls with $\mu_{\realbest} = 8$, $\Delta = 3$ and
  $T_S = 200$, $500$ hypotheses with $200$ true nulls and $30$ arms,
  the desired mFDR level is $\alpha = 0.1$ (b) mFDR over different
  proportions of non-nulls $\pi_1$, with same settings, averaged over
  $80$ runs.}
\label{FigFDR}
\vspace{-0.05in}
\end{figure}

\vspace{-0.1in}
\section{Proofs}

In this section we provide the proofs of the main results in the paper.

\vspace{-0.1in}
\subsection{Proof of Proposition~\ref{PropPVal}}
\label{SecProofPropPVal}

For any fixed $\gamma \in (0,1)$, we have the equivalence
\begin{align*}
\widehat{\mu}_{i,n_i(t)} - \dev_{n_i(t)}(\tfrac{\gamma}{2\numarms}) >
\widehat{\mu}_{0,n_0(t)} +\dev_{n_0(t)}(\tfrac{\gamma}{2}) + \epsilon
\quad \Longleftrightarrow \quad p_{i,t} \leq \gamma.
\end{align*}
If $\max \limits_{i=1,\dots,K} \mu_i \leq \mu_0 + \epsilon$, then we
have
\begin{align*}
\hspace{.5in}&\hspace{-.5in}\PP\left( \bigcup_{i=1}^K
\bigcup_{t=1}^\infty \Big\{ \widehat{\mu}_{i,n_i(t)} -
\dev_{n_i(t)}(\tfrac{\gamma}{2\numarms}) > \widehat{\mu}_{0,n_0(t)}
+\dev_{n_0(t)}(\tfrac{\gamma}{2}) + \epsilon \Big\} \right)\\
& = 1 - \PP\left( \bigcap_{i=1}^K \bigcap_{t=1}^\infty \Big\{
\widehat{\mu}_{i,n_i(t)} - \dev_{n_i(t)}(\tfrac{\gamma}{2\numarms})
\leq \widehat{\mu}_{0,n_0(t)} +\dev_{n_0(t)}(\tfrac{\gamma}{2}) +
\epsilon \Big\} \right)\\
& \leq 1- \PP\left( \bigcap_{t=1}^\infty
\Big\{ \mu_0 \leq \widehat{\mu}_{0,t} +\dev_{t}(\tfrac{\gamma}{2})
\Big\} \cap \bigcap_{i=1}^K \bigcap_{t=1}^\infty \Big\{
\widehat{\mu}_{i,n_i(t)} - \dev_{n_i(t)}(\tfrac{\gamma}{2\numarms})
\leq \mu_i \Big\} \right) \\
& \leq \PP\left( \bigcup_{t=1}^\infty \Big\{ \mu_0 >
\widehat{\mu}_{0,t} +\dev_{t}(\tfrac{\gamma}{2}) \Big\} \right) +
\sum_{i=1}^K \PP\left( \bigcup_{t=1}^\infty \Big\{
\widehat{\mu}_{i,n_i(t)} - \dev_{n_i(t)}(\tfrac{\gamma}{2\numarms}) >
\mu_i \Big\} \right) \\
& \leq \tfrac{\gamma}{2} + K \tfrac{\gamma}{2K} = \gamma
\end{align*} 
by equation~\eqref{eqn:lil_confidence}.  Thus, we have
$\PP\left(\bigcup_{i=1}^K \bigcup_{t=1}^\infty \Big\{ p_{i,t} \leq
\gamma \Big\} \right) \leq \gamma$, which completes the proof.


\subsection{Proof of Proposition~\ref{PropLUCBEps}}
\label{SecProofPropLUCBEps}

Here we prove that the algorithm~\ref{AlgoModLUCB} terminates in finite time. The technical proof for sample complexity is moved to the Appendix~\ref{SecSampleComplex}.
It suffices to argue for $\delta/2 \leq 0.1$ and we discuss the other case
at the end. 

\paragraph{Proof of termination in finite time}

First we prove by contradiction that the algorithm terminates in
finite time with probability one for the case
$\mu_0 \geq \max_{i=1,\dots,K} \mu_i - \epsilon$.

Assuming that there exist
runs for which the algorithm does not terminate, the set of arms
defined by
\begin{equation*}
S := \{ i: \LCB_0(t) \leq \UCB_i(t) -
\epsilon \text{ infinitely often (i.o.)}\}
\end{equation*}
is necessarily non-empty for these runs. We now show
that this assumption yields a contradiction so that 
\begin{align}
\label{EqnTerminates}
\PP(\text{Algorithm does not terminate}) \leq 
\PP( \LCB_0(t) \leq \max_{i=1,\dots, \numarms} \UCB_i(t) -
\epsilon \text{ i.o.}) = 0
\end{align}

First take note that by definition of the algorithm, if an arm $i$ is
drawn infinitely often (i.o.), then so is the control arm $0$ and we have
$\LCB_0(t) \to \mu_0$ as well as $\UCB_i(t) \to \mu_i$ as $t\to \infty$. 
This follows by the law of large numbers combined
with the fact that $\dev_{n_i(t)}, \dev_{n_0(t)} \to 0$ as $t\to \infty$,
since $\dev_{n} \to 0$ as $n\to \infty$.
Since for the null hypothesis we have $\mu_0 > \mu_i - \epsilon$, 
it follows that $\LCB_0(t) > \UCB_i(t) - \epsilon$ for all $t\geq t'$ 
for some $t'$. 

This argument implies that all arms $i \in S$ can only be drawn a
finite number of times, i.e. $n_i(t) <\infty$ for all $i \in S$.
However, the fact that they are not drawn i.o. implies that $h_t \neq
i$ and $\ell_t \neq i$ i.o. for all $i\in S$, so that
there exists $i' \not \in S$ such that $\max_{i \in S}\UCB_i(t) \leq
\UCB_{i'}(t)$ i.o. By definition of $S$ we then obtain
\begin{equation}
\label{EqnContradict}
\LCB_0(t) \leq \UCB_{i'}(t) - \epsilon \text{ i.o. }
\end{equation}
However, since $i' \not \in S$, inequality \eqref{EqnContradict} cannot hold
and equation~\eqref{EqnTerminates} is proved.

A nearly identical argument to the above shows that the stopping
condition is met in finite time.


\subsection{Proof of Theorem~\ref{ThmOnlinemFDR}}
\label{SecProofThm1}

We now turn to the proof of Theorem~\ref{ThmOnlinemFDR}, splitting our
argument into parts (a) and (b), respectively.


\subsubsection{Proof of part (a)}
\label{SecProof1a}

In order for generalized alpha-investing procedures such as LORD to
successfully control the mFDR, it is sufficient that $p$-values under
the null be \emph{conditionally super-uniform}, meaning that for all
$j \in \truenulls$, we have
\begin{align}
\label{EqnPCondition}
\PP_{0}(\nullp{j} \leq \alpha_j |\sfield^{j-1}) \leq
\alpha_j(\rej_1,\dots, \rej_{j-1})
\end{align}
where $\sfield^{j-1}$ is the $\sigma$-field induced by $\rej_1, \dots,
\rej_{j-1}$.  
Note that as long as condition~\eqref{EqnPCondition} is satisfied, 
$\rowstop_j$ and thus
$\nullp{j}$ could potentially depend on $\alpha_j$, i.e. the rejection
indicator variables $R_1,\dots, R_{j-1}$ and potentially
$\nullp{1},\dots, \nullp{j-1}$.  See Aharoni and Rosset~\cite{AR14}
for further details.  

It thus suffices to show that condition~\eqref{EqnPCondition} holds
for our definition of $p$-value in our framework.  We know that by
Proposition~\ref{PropPVal} we have for any random stopping time, thus
any fixed truncation time $\trunctime$, that $\PP_0(\nullp{j}_T \leq
\alpha_j)\leq\alpha_j$. We now show that the same bound also holds for
the ($\alpha_j$-dependent) bandit stopping time $T(\putinalpha)$, i.e.
that $\PP_0(\nullp{j}_{T(\putinalpha)} \leq\alpha_j) \leq \alpha_j$.

Under the null hypothesis, the best arm is at most $\epsilon$ better
than the control arm, i.e. $\mu_0 > \mu_i - \epsilon$, so that by 
Proposition~\ref{PropLUCBEps} we have that with probability $\geq 1-\alpha_j$,
$\Output = 0$, i.e. $\LCB_0(t) > \UCB_i(t) - \epsilon$ for all $i\neq 0$.
Hence, $\LCB_i(t) - \UCB_0(t) < \epsilon$, and thus, by the definition of the $p$-values,
$\nullp{j}_{i,\banditstop} = 1$ for all $i$ with probability $\geq 1-\alpha_j$. 
It finally follows that $\PP_0(\nullp{j}_{\banditstop} \leq \alpha_j) \leq \alpha_j$. 

Putting things together, under the true null hypothesis (omitting the index
$j\in \truenulls$ to simplify notation) we directly have that for any
$\alpha_j$
\begin{align*}
\PP_0(\nullp{j}_{\rowstop_j} (\alpha_j)\leq \alpha_j) &=
\PP_0\big(\nullp{j}_{\banditstop} \leq \alpha_j\big|
\banditstop \leq \trunctime\big)\PP_0(\banditstop\leq
\trunctime) \\ &+\PP_0\big(\nullp{j}_{\trunctime} \leq \alpha_j\big|
\banditstop > \trunctime\big) \PP_0(\banditstop >
\trunctime) \\ & \leq \alpha_j ( \PP_0(\banditstop\leq
\trunctime) + \PP_0(\banditstop > \trunctime)) = \alpha_j
\end{align*}
for all fixed $\alpha_j$ even when the stopping time
$\banditstop$ is dependent on $\alpha_j$.  This is
equivalent to stating that for any sequence $\rej_1, \dots,
\rej_{j-1}$ we have
\begin{align*}
\PP_{0}(\nullp{j} \leq \alpha_j(\rej_1, \dots, \rej_{j-1})
  |\sfield^{j-1}) &=
  \PP_{0}(\nullp{j}_{T(\alpha_j(\rej_1,\dots,\rej_{j-1}))} \leq
  \alpha_j(\rej_1, \dots, \rej_{j-1})) \\
  &\leq \alpha_j(\rej_1,\dots,
  \rej_{j-1})
\end{align*}
and the proof is complete.


\subsubsection{Proof of part (b)}

It suffices to prove that for a single experiment $j$ and $\trunctime
= \infty$, we have $\PP_1(\nullp{j}_{\banditstop}\leq \alpha_j) \geq
1-\alpha_j$ where $\PP_1$ is the distribution of a non-null experiment $j$. 
First observe that at stopping time $\banditstop$ of
Algorithm~\ref{AlgoModLUCB}, either $\nullp{j}_{i,\banditstop} \leq \alpha_j$
or $\nullp{j}_{i,\banditstop} = 1$ for all $i$.  The former event happens 
whenever the algorithm exits with $\Output \in \SetStar$, i.e. 
when $\LCB_{\Output}(t) \geq \UCB_{\ell_t}(t) - \epsilon$ holds. 
Then, by definition
of the $p$-value in equation~\eqref{EqnPVal} and $\ell_t$ we must have that
$\nullp{j}_{\Output,\banditstop} \leq \putinalpha$. 
 As a consequence, by Proposition~\ref{PropLUCBEps}, we have
\begin{align*}
\PP_1 (\nullp{j}_{\banditstop} \leq \alpha_j) &\geq \PP(\nullp{j}_{\banditstop} \leq \putinalpha)\\
&\geq \PP_1(\text{Algorithm~\ref{AlgoModLUCB}
  exits with} \: \Output \in \SetStar) \\
  &\geq 1 - \alpha_j
\end{align*}
and the proof is complete.

\section{Discussion}
\label{SecDiscussion}

The recent focus in popular media about the lack of reproducibility of
scientific results erodes the public's confidence in published
scientific research.  To maintain high standards of published results
and claimed discoveries, simply increasing the statistical
significance standards of each individual experimental work (e.g.,
reject at level 0.001 rather than 0.05) would drastically hurt power.
We take the alternative approach of controlling the ratio of false
discoveries to claimed discoveries at some desired value (e.g., 0.05)
over many sequential experiments.  This means that the statistical
significance for validating a discovery changes from experiment to
experiment, and could be larger or smaller than 0.05, requiring less
or more data to be collected.  Unlike earlier works on online FDR
control, our framework synchronously interacts with adaptive sampling
methods like MABs over uniform sampling to make the overall sampling
procedure as efficient as possible.  We do not know of other works in
the literature combining the benefits of adaptive sampling and FDR
control.  It should be clear that any improvement, theoretical or
practical, to either online FDR algorithms or best-arm identification
in MAB (or their variants), immediately results in a corresponding
improvement for our MAB-FDR framework.

More general notions of FDR with corresponding online
procedures have recently been developed by Ramdas et
al~\cite{onlineFDR17}. In particular, they incorporate the notion of
memory and a priori importance of each hypothesis. This could prove to be a
valuable extension for our setting, especially in cases when only the
percentage of wrong rejections in the recent past matters. It would
be useful to establish FDR control for these generalized notions of FDR
as well.

There are several directions that could be explored in future
work. First, it would be interesting to extend the MAB aspect (in
which each arm is univariate) of our framework to more general
settings. Balasubramani and Ramdas~\cite{BR16} show how to construct
sequential tests for many multivariate nonparametric testing problems,
using LIL confidence intervals, which can again be inverted to provide
always valid p-values. It might be of interest to marry the ideas in
our paper with theirs. For example, the null hypothesis might be that
the control arm has the same (multivariate) mean as other arms
($K$-sample testing), and under the alternative, we would like to pick
the arm whose mean is furthest away from the control. A more
complicated example could involve dependence, where we observe pairs
of arms, and the null hypothesis is that the rewards in the control
arm are independent of the alternatives, and if the null is false we
may want to pick the most correlated arm.  The work by Zhao et
al.~\cite{Ermon16} on tightening LIL-bounds could be practically
relevant.  Recent work on sequential p-values by Malek et
al.~\cite{Malek17}  also naturally fit
into our framework.  Lastly, in this work we treat samples or pulls
from arms as identical from a statistical perspective; it might be of
interest in subsequent work to extend our framework to the contextual
bandit setting, in which the samples are associated with features to
aid exploration.


\subsection*{Acknowledgements}

This work was partially supported by Office of Naval Research MURI
grant DOD-002888, Air Force Office of Scientific Research Grant
AFOSR-FA9550-14-1-001, and National Science Foundation Grants
CIF-31712-23800 and DMS-1309356.

\bibliography{fdrbandit}

\begin{thebibliography}{10}
\providecommand{\url}[1]{#1}
\csname url@rmstyle\endcsname
\providecommand{\newblock}{\relax}
\providecommand{\bibinfo}[2]{#2}
\providecommand\BIBentrySTDinterwordspacing{\spaceskip=0pt\relax}
\providecommand\BIBentryALTinterwordstretchfactor{4}
\providecommand\BIBentryALTinterwordspacing{\spaceskip=\fontdimen2\font plus
\BIBentryALTinterwordstretchfactor\fontdimen3\font minus
  \fontdimen4\font\relax}
\providecommand\BIBforeignlanguage[2]{{%
\expandafter\ifx\csname l@#1\endcsname\relax
\typeout{** WARNING: IEEEtran.bst: No hyphenation pattern has been}%
\typeout{** loaded for the language `#1'. Using the pattern for}%
\typeout{** the default language instead.}%
\else
\language=\csname l@#1\endcsname
\fi
#2}}

\bibitem{BH95}
Y.~Benjamini and Y.~Hochberg, ``Controlling the false discovery rate: a
  practical and powerful approach to multiple testing,'' \emph{Journal of the
  Royal Statistical Society. Series B (Methodological)}, pp. 289--300, 1995.

\bibitem{FS08}
D.~P. Foster and R.~A. Stine, ``$\alpha$-investing: a procedure for sequential
  control of expected false discoveries,'' \emph{Journal of the Royal
  Statistical Society: Series B (Statistical Methodology)}, vol.~70, no.~2, pp.
  429--444, 2008.

\bibitem{AR14}
E.~Aharoni and S.~Rosset, ``Generalized $\alpha$-investing: definitions,
  optimality results and application to public databases,'' \emph{Journal of
  the Royal Statistical Society: Series B (Statistical Methodology)}, vol.~76,
  no.~4, pp. 771--794, 2014.

\bibitem{JM16}
A.~Javanmard and A.~Montanari, ``Online rules for control of false discovery
  rate and false discovery exceedance,'' \emph{The Annals of Statistics}, 2017.

\bibitem{JPW15}
R.~Johari, L.~Pekelis, and D.~J. Walsh, ``Always valid inference: Bringing
  sequential analysis to {A/B} testing,'' \emph{arXiv preprint
  arXiv:1512.04922}, 2015.

\bibitem{JMNB14}
K.~G. Jamieson, M.~Malloy, R.~D. Nowak, and S.~Bubeck, ``lil'{UCB}: An optimal
  exploration algorithm for multi-armed bandits,'' in \emph{COLT}, vol.~35,
  2014, pp. 423--439.

\bibitem{BR16}
A.~Balsubramani and A.~Ramdas, ``Sequential nonparametric testing with the law
  of the iterated logarithm,'' in \emph{Proceedings of the Thirty-Second
  Conference on Uncertainty in Artificial Intelligence}.\hskip 1em plus 0.5em
  minus 0.4em\relax AUAI Press, 2016, pp. 42--51.

\bibitem{KCG15}
E.~Kaufmann, O.~Capp{\'e}, and A.~Garivier, ``On the complexity of best arm
  identification in multi-armed bandit models,'' \emph{The Journal of Machine
  Learning Research}, 2015.

\bibitem{jamieson2014best}
K.~Jamieson and R.~Nowak, ``Best-arm identification algorithms for multi-armed
  bandits in the fixed confidence setting,'' in \emph{Information Sciences and
  Systems (CISS), 2014 48th Annual Conference on}.\hskip 1em plus 0.5em minus
  0.4em\relax IEEE, 2014, pp. 1--6.

\bibitem{villar15}
S.~S. Villar, J.~Bowden, and J.~Wason, ``Multi-armed bandit models for the
  optimal design of clinical trials: benefits and challenges,''
  \emph{Statistical science: a review journal of the Institute of Mathematical
  Statistics}, vol.~30, no.~2, p. 199, 2015.

\bibitem{KTAS12}
S.~Kalyanakrishnan, A.~Tewari, P.~Auer, and P.~Stone, ``Pac subset selection in
  stochastic multi-armed bandits,'' in \emph{Proceedings of the 29th
  International Conference on Machine Learning (ICML-12)}, 2012, pp. 655--662.

\bibitem{SJR17}
M.~Simchowitz, K.~Jamieson, and B.~Recht, ``The simulator: Understanding
  adaptive sampling in the moderate-confidence regime,'' \emph{arXiv preprint
  arXiv:1702.05186}, 2017.

\bibitem{JM15}
A.~Javanmard and A.~Montanari, ``On online control of false discovery rate,''
  \emph{arXiv preprint arXiv:1502.06197}, 2015.

\bibitem{onlineFDR17}
A.~Ramdas, F.~Yang, M.~J. Wainwright, and M.~I. Jordan, ``Online control of the
  false discovery rate with decaying memory,'' in \emph{Advances in {N}eural
  {I}nformation {P}rocessing {S}ystems (NIPS) 2017, arXiv preprint
  arXiv:1710.00499}, 2017.

\bibitem{Ermon16}
S.~Zhao, E.~Zhou, A.~Sabharwal, and S.~Ermon, ``Adaptive concentration
  inequalities for sequential decision problems,'' in \emph{Advances In Neural
  Information Processing Systems}, 2016, pp. 1343--1351.

\bibitem{Malek17}
A.~Malek, Y.~Chow, M.~Ghavamzadeh, and S.~Katariya, ``Sequential multiple
  hypothesis testing with type {I} error control,'' in \emph{The 20th
  International Conference on Artificial Intelligence and Statistics, 2017},
  2017, pp. 1343--1351.

\end{thebibliography}
\bibliographystyle{IEEEtran}

\newpage
\appendix

\section{Notation}
\begin{table}[h!]
\centering
\begin{tabular}{| l l |}
\hline
Notation & Terminology and explanation \\
\hline
$\MAB$ & (pure exploration for best-arm identification in) multi-armed bandits\\
$\FDR(J)$ & the expected ratio of \# false discoveries to \# discoveries up to experiment $J$\\
$\mFDR(J)$ & the ratio of expected \# false discoveries to expected \# discoveries \\
$\alpha$ &  target for $\FDR$ or $\mFDR$ control after any number of experiments\\
$\BDR(J)$ & the best arm discovery rate (generalization of test power)\\
$\epsBDR(J)$ & the $\epsilon$-best arm discovery rate (softer metric than $\BDR$) \\
$\LCB,\UCB$ & the lower and upper confidence bounds used in the best-arm algorithms\\
\hline
$j \in \NN$ & experiment counter (number of MAB instances)\\
$T_j \in \NN$ & stopping time for the $j$-th experiment \\
$P^j_t, P_t \in [0,1]$ & always valid $p$-value after time $t$ (in experiment $j$, explicit or implicit)\\
$\nullp{j}$ & always valid $p$-value for experiment $j$ at its stopping time $T_j$\\
$\alpha_j \in [0,1]$ & threshold set by the online FDR algorithm for $P^j$, using $\{p_i\}_{i=1}^{j-1}$\\
$T(\alpha_j) \in \NN$ & stopping time for the $j$-th experiment, when experiment uses $\alpha_j$\\
\hline
0 & the control or default arm\\
$\{1,\dots,\numarms\}$ & $\numarms = \numarms(j)$ alternatives or treatment arms (experiment $j$ implicit)\\
$i \in \{0,\ldots,\numarms\}$ & $\numarms+1$  options or ``all arms''\\
$\realbest, \Output$ & the best of all arms, and the arm returned by MAB \\
$\mu_i, \mu_*$ & the mean of the $i$-th arm, and the mean of the best arm \\
$t, n_i(t) \in \NN$ & total number of pulls, number of times arm $i$ is pulled up to time $t$\\
\hline
\end{tabular}
\caption{Common notation used throughout the paper.}
\label{tab:notation}
\end{table}

\section{Notes on FDR control}
\label{SecFDRControl}
We can prove FDR control for our framework using the specific online
FDR procedure called LORD '15 introduced in \cite{JM15}.  When used in
Procedure~\ref{ProcLORD}, the only adjustment that needs to be made is
to reset $W(j+1)$  to $\alpha$ in step 2 after every rejection, yielding
$\alpha_j = \alpha \gamma_{j - \tau_j}$ for any sequence $\{\gamma_j\}_{j=1}^\infty$
such that $\sum_{j=1}^\infty \gamma_j = 1$. We call the adjusted
procedure MAB-LORD' for short.

\begin{theos}[Online FDR control for MAB-LORD]
\begin{enumerate}[(a)]
\item MAB-LORD' achieves mFDR and FDR control at a specified level
$\alpha$ for stopping times $T_j = \min\{T(\alpha_j), M\}$.
\item Furthermore, if we set $\trunctime = \infty$,
  MAB-LORD' satisfies
\begin{align}
\epsBDR(\numexp) \geq \frac{  (1- \alpha)}{|\falsenulls(\numexp)|}.
\end{align}
\end{enumerate}
\end{theos}

Note that LORD as in \cite{JM15} is less powerful than in \cite{JM16}
since the values $\alpha_j$ in the former can be much smaller than
those in \cite{JM16}, which could in fact exceed the level $\alpha$.
Therefore, for FDR control we currently do have to sacrifice some
power.

\begin{proof}
We leverage the proposition that can be obtained from a slightly more
careful analysis of the procedure than in \cite{JM15}.
\begin{props}
\label{PropFDRLORD}
If $\PP_0 (\nullp{j} \leq \alpha_j \mid \tau_{j}) \leq \alpha_j$, i.e.
the distribution of the $p-$values under the null are superuniform
conditioned on the last rejection, using the online LORD'15 procedure
controls the FDR at each $t$.
\end{props}

Note that this proposition allows online FDR control for any, possibly
dependent, $p$-values which are conditionally superuniform.  This
condition is not equivalent to~\eqref{EqnPCondition} in general, it is
in fact less restrictive since the probability is conditioned only on a function
$\tautil_j = \max \{k \leq j: R_k = 1\}$ of all past
rejections. Formally, the sigma algebra induced by $\tau_{j-1}$ is contained
in $\sfield^{j-1}$ and hence $\PP_0(\nullp{j}\leq \alpha_j \mid \tau_{j-1}) \leq \PP_0(\nullp{j} \leq \alpha_j \mid R_1,\dots,R_j)$ by the tower property.
Finally, utilizing the fact that our $p$-values are conditionally
super-uniform as proven in Section~\ref{SecProof1a},
i.e. inequality~\eqref{EqnPCondition} holds, the condition for
Proposition~\ref{PropFDRLORD} is fulfilled and the proof is complete.
\end{proof}

\subsection{Proof of Proposition~\ref{PropFDRLORD}}
Let $\tautil_i$ denote the time of the $i$-th rejection with
$\tautil_0 = 0$ (note that this is different from $\tau_j$). 
and define $k(t) = \sum_{j=1}^t R_j$. 
Let $H_j$ be the $j-$th hypothesis that was rejected.
We adjust an argument from~\cite{JM15}.

First observe that
$\{ k(t) = \ell \} = \{ \tautil_{\ell} \leq t, \tautil_{\ell+1} > t\}$ 
and $FDP(t) = FDP(\tautil_{k(t)})$ so that 
\begin{align*}
\EE FDP(t) &= \EE FDP(\tau_{k(t)}) = \sum_{\ell=1}^t \EE \big[
  \frac{\sum_{j\in \truenulls}R_j}{\ell} \mid k(t) = \ell\big] P(k(t)
= \ell)\\ 
&= \sum_{\ell=1}^t P(k(t) = \ell) \sum_{i=1}^\ell \EE \big[
  \frac{\Indi_{H_i \in \truenulls}}{\ell} \mid k(t) = \ell\big]\\ 
&=  \sum_{\ell=1}^t P(k(t) = \ell) \sum_{i=1}^\ell \EE \big[ \EE \big(
    \frac{\sum_{j = \tautil_{i-1}+1}^{\tautil_i}   R_j\Indi_{j\in \truenulls}
  }{\ell} \mid \tautil_0,\dots,\tautil_{i-1}\big) \mid \tautil_{\ell}\leq t,
  \tautil_{\ell+1} > t\big]
\end{align*}
 
Since for the LORD '15 procedure, we have $\alpha_t = \gamma_{t-\tau_t}$,
and thus for all positive integers $i$, the random variables $R_j $
with $j\geq \tautil_{i-1}$ are conditionally independent of
$\tautil_0,\dots, \tautil_{i-2}$ given $\tautil_{i-1}$.  Additionally
noting that $\tautil_{i-1} = \tau_j$ for all $j\geq \tautil_{i-1}$ by
definition of $\tautil$ and $\tau$, using $\EE_0 (\Indi_{p_j \leq
  \alpha_j} \mid \tau_{j}) \leq \alpha_j$ we obtain
\begin{align*}
 \EE \big(
    \frac{\sum_{j \in (\tautil_{i-1},\tautil_i] \bigcap j\in \truenulls} R_j
  }{\ell} \mid \tautil_0,\dots,\tautil_{i-1}\big) &=  \EE \big(
    \frac{\sum_{j = \tautil_{i-1}+1}^{\tautil_i}   R_j\Indi_{j\in \truenulls}
  }{\ell}  \mid \tautil_{i-1}\big) \\
&\leq 
\frac{ \sum_{j=\tau_{i-1}+1}^{\tau_i} \Indi_{j\in H_0} \EE [R_j\mid \tau_j]}{\ell} \\
&\leq \frac{\sum_{j=\tau_{i-1}+1}^{\tau_i} \alpha_j}{\ell} \leq \frac{\alpha}{\ell}.
\end{align*}
The last inequality follows since between any two rejection times $\tau_k, \tau_{k+1}$,  we have
\begin{equation*}
\sum_{i=\tau_k}^{\tau_{k+1}} \alpha_i\leq \alpha\sum_{i=1}^\infty \gamma_i \leq \alpha.
\end{equation*}
Since $\sum_{\ell=1}^t P(k(t) = \ell) =1$ it follows that FDR control is obtained.



\section{Proof of sample complexity for Proposition~\ref{PropLUCBEps}}
\label{SecSampleComplex}

In the sequel we use $\gtrsim, \sim$ for inequality and equality up to constant factors.

Define $\realbest = \arg\max_{i=0,1,\dots,K} \mu_i$ (breaking ties arbitrarily) and $n_i(t)$ to be
the number of times sample $i$ was drawn until time $t$.
For any $i \in \{0,1,\dots,K\}$ and $\eta \in \mathbb{R}$ we define the following key quantity 
\begin{align}
\tau_i(\eta,\xi) &:= \min\{ n \in \mathbb{N} : 2\dev_{n}(\tfrac{\delta}{2K}) < \max\{ |\eta - \mu_i|, \xi\} \}  \label{eqn:tau_i_def} \\
&\lesssim \min\left\{ (\eta-\mu_i)^{-2} \log(K \log (\eta-\mu_i)^{-2})/\delta),  \xi^{-2} \log(K \log( \xi^{-2})/\delta) \right\} \nonumber
\end{align}
where we set $\tau_i(\mu_i, 0) = \infty$, but this case does not arise in our analysis.

Let us define the events
\begin{align*}
\mathcal{E}_i =  \bigcap_{n=1}^\infty \{ |\widehat{\mu}_{i,n} - \mu_{i}| \leq \dev_{n}(\tfrac{\delta}{2K}) \}.
\end{align*}
By a union bound and the LIL bound in~\eqref{eqn:lil_confidence}, we have for $\delta/2K < 0.1$ that $\mathbb{P}\left(  \bigcup_{i=0}^K \mathcal{E}_i^c \right) \leq \frac{K+1}{2K}\delta \leq \delta$
for $K\geq 2$.
For $\tfrac{\delta}{2\numarms} > 0.1$, note
that for all $\delta' < \delta$ we have $\dev_n(\delta') \leq
\dev_n(\delta)$ so that
\begin{align*}
\PP(\event_i^c) &= \PP(\dev_n(\tfrac{\delta}{2\numarms}) < \empmu_{i,n} - \mu_i )\\
& \leq   \PP(\dev_n(0.1) < \empmu_{i,n} - \mu_i) \leq \tfrac{\delta}{2\numarms} \qquad \forall i = 1,\dots, \numarms
\end{align*}
Throughout the rest of the proof we assume the events $\mathcal{E}_i$ hold.

The following simple lemma regarding the key quantity $\tau_i$ will be used throughout the proof.
\begin{lems}
  \label{Lem1}
Fix $i \in \{0,1,\dots,K\}$ and $\eta > 0$. 
For any $t \in \mathbb{N}$, whenever $n_i(t) \geq \tau_i(\eta,\xi)$ we have that under the event $ \bigcap_{i=0,\dots,K} \mathcal{E}_i$, we have
\begin{align*}
\UCB_i(t) \leq \max\{\eta, \mu_i + \xi\} \text{ if } \eta \geq \mu_i \\
\LCB_i(t) \geq \min\{\eta, \mu_i - \xi\} \text{ if } \eta \leq \mu_i
\end{align*}
\end{lems}
\begin{proof}
Assume $n_i(t) \geq \tau_i(\eta,\xi)$.
If $\eta \geq \mu_i$ we have by definition of $\mathcal{E}_i$ that
\begin{align*}
\UCB_i(t) = \widehat{\mu}_{i,n_i(t)} + \dev_{n_i(t)}(\tfrac{\delta}{2}) &\leq \mu_i + 2\dev_{n_i(t)}(\tfrac{\delta}{2K})< \mu_i + \max\{ \eta - \mu_i, \xi\} 
\end{align*}
and if $\eta \leq \mu_i$
\begin{align*}
\LCB_i(t) = \widehat{\mu}_{i,n_i(t)} - \dev_{n_i(t)}(\tfrac{\delta}{2K}) &\geq \mu_i - 2\dev_{n_i(t)}(\tfrac{\delta}{2K}) > \mu_i - \max\{ \mu_i - \eta, \xi\} = \mu_i + \min\{ \eta -\mu_i, -\xi\}
\end{align*}
\end{proof}


\subsection{ Proof of Proposition~\ref{PropLUCBEps} (a)   $\mu_0 > \max \limits_{i=1,\dots,K} \mu_i - \epsilon$}
At each time $t$ which does not satisfy the stopping condition, arm $0$ and $\arg\max_{i=1,\dots,K} \UCB_i(t)$ are pulled.
Note that by Lemma~\ref{Lem1}
\begin{align}
\{ n_0(t) \geq \tau_0(\tfrac{\mu_0 + (\max \limits_{i=1,\dots,K} \mu_i - \epsilon)}{2},0) \} \implies  \LCB_0(t) &\geq \min\{ \tfrac{\mu_0 + (\max \limits_{i=1,\dots,K} \mu_i - \epsilon)}{2}, \mu_0\} \geq \tfrac{\mu_0 + (\max \limits_{i=1,\dots,K} \mu_i - \epsilon)}{2} \label{eqn:case1_i} 
\end{align}
so that $t> n_0(t)$ makes sure that there were enough draws for the particular arm $0$ (since it's
drawn every time).
For $i \neq 0$ we have
\begin{align}
\{ n_i(t) \geq \tau_i(\tfrac{(\mu_0 + \epsilon) + \max \limits_{i=1,\dots,K} \mu_i}{2},0) \} \implies  \UCB_i(t) &\leq \max\{ \tfrac{(\mu_0 + \epsilon) + \max \limits_{i=1,\dots,K} \mu_i}{2}, \mu_i \} \leq \tfrac{(\mu_0 + \epsilon) + \max \limits_{i=1,\dots,K} \mu_i}{2} . \label{eqn:case1_ii}
\end{align}
which makes $t> \sum_{i=0}^K n_i(t)$ a necessary condition.

Reversely whenever $t> \sum_{i=0}^Kn_i(t)$, for all arms $i \neq 0$ we have $\UCB_i(t) \leq \tfrac{(\mu_0 + \epsilon) + \max \limits_{i=1,\dots,K} \mu_i}{2}$.  In essence, once arm $i$ has been sampled $n_i(t)$ times, because of \eqref{eqn:case1_ii}, it will not be sampled again  - either, because all of the other $UCB_i(t)$ satisfy the same upper bound, the algorithm will have stopped, or, if for some $i$ we have $UCB_i(t) > \tfrac{(\mu_0 + \epsilon) + \max \limits_{i=1,\dots,K} \mu_i}{2}$ that will be the arm that is drawn.
Thus, 
\begin{align*}
  &\{ t \geq \UpperSample_1(\mu,\delta) \defn \tau_0(\tfrac{\mu_0 + (\max \limits_{i=1,\dots,K} \mu_i - \epsilon)}{2},0) + \sum_{i=1}^K \tau_i(\tfrac{(\mu_0 + \epsilon) + \max \limits_{i=1,\dots,K} \mu_i}{2},0) \} \\
&\implies \{ \LCB_{0}(t) - \UCB_{i}(t) \geq - \epsilon \quad \forall i\neq 0\},
\end{align*}
i.e., the stopping condition is met, where the first term accounts for satisfying \eqref{eqn:case1_i}, the second term accounts for satisfying \eqref{eqn:case1_ii} for all $i\neq 0$, and the third term accounts for satisfying Equation~\eqref{eqn:prelim_stopping_condition}. Denoting $T(\delta)$ as the stopping time
of the algorithm, this implies that with probability at least $1-\delta$, we have $T(\delta) \leq \UpperSample_1(\mu,\delta)$ and arm $0$ is returned.

Let us now simplify the expression to make it more accessible to the reader and arrive at the theorem
statement.
Defining $\effgap_i \defn \max\{|\eta-\mu_i|, \xi\}$ as the \emph{effective gap} in the definition of $\tau_i(\eta, \xi)$ in Equation~\eqref{eqn:tau_i_def}, it is straightforward to verify that the effective gap associated with arm $0$ is equal to
\begin{align*}
  \effgap_0
           \sim (\mu_0 + \epsilon) - \max \limits_{j=1,\dots,K} \mu_j,
\end{align*}
and the effective gap for any other arm $i$ is equal to
\begin{align*}
  \effgap_i
    \gtrsim (\mu_0 + \epsilon) -  \mu_i.
\end{align*}
Using these quantities, we can see that the upper bound $\UpperSample_1(\mu, \delta)$ scales like $\sum_{i=0}^K \effgap_i^{-2} \log(K \log(\effgap_i^{-2})/\delta)$.\\

\subsection{Proof of Proposition~\ref{PropLUCBEps} (b) $\max
  \limits_{i=1,\dots,K} \mu_i = \mu_{\realbest} > \mu_0 + \epsilon$}
At each time $t$ which does not satisfy the stopping condition, arm $0$ is pulled.
Note again that by Lemma~\ref{Lem1}
\begin{align*}
\{ n_0(t) \geq \tau_0(\tfrac{(\mu_{\realbest}-\epsilon)+\mu_0}{2},0) \} \implies  \UCB_0(t) &\leq \max\{\tfrac{(\mu_{\realbest}-\epsilon)+\mu_0}{2}, \mu_0\} \leq \frac{(\mu_{\realbest}-\epsilon)+\mu_0}{2}.
\end{align*}

The following claim is key to proving this case (where $u \in (0,1)$ be an absolute constant to be defined later).
\begin{claim}
  \label{ClaimKevin}
  Under the event $\bigcap_{i=0,\dots,K} \mathcal{E}_i$, for any $u \leq \tfrac{2}{7}$ and $\barmu \in [\max_{j \neq \realbest} \mu_j, \mu_{\realbest}]$, we have
\begin{align}
|\{ s \geq 2\sum_{i=0}^K \tau_i(\barmu, u\epsilon) : \LCB_{h_s}(s) \leq \mu_{\realbest} - \tfrac{5}{2}u\epsilon \text{ or } \UCB_{\ell_s}(s) \geq \mu_{\realbest} + u\epsilon \}| < \sum_{i=0}^K \tau_i(\barmu, u\epsilon) \label{eqn:prelim_stopping_condition}
\end{align}
\end{claim}
The proof of this claim can be found in Appendix~\ref{SecClaimProof}.
Note that for all $s$ we have that
\begin{equation*}
  \LCB_{h_s}(s) \geq \mu_{\realbest} - \tfrac{5}{2}u\epsilon \text{ and } \UCB_{\ell_s}(s) \leq \mu_{\realbest} + u\epsilon \implies \LCB_{h_s}(s) \geq \UCB_{\ell_s}(s) - \epsilon.
\end{equation*}
Intuitively the inequality~\eqref{eqn:prelim_stopping_condition} thus limits the number of times that for $t \geq 2\sum_{i=0}^K \tau_i(\barmu, u\epsilon)$, the criterion $\LCB_{h_s}(s) \geq \UCB_{\ell_s}(s) - \epsilon$ is not fulfilled. We refer to the times when the condition on the left hand side of inequality~\eqref{eqn:prelim_stopping_condition} is fulfilled, as ``good'' times.

Applying Claim~\ref{ClaimKevin} with
$\barmu = \max_{j \neq \realbest}\tfrac{\mu_{\realbest}+\mu_j}{2}$ and
$u=\frac{\mu_{\realbest} - (\mu_0+\epsilon)}{5\epsilon}$ we then observe
that on the ``good'' times, we have
\begin{align*}
\LCB_{h_t} \geq \mu_{\realbest} - \tfrac{5}{2} u \epsilon = \frac{\mu_{\realbest} + (\mu_0+\epsilon)}{2} = \frac{(\mu_{\realbest}-\epsilon)+\mu_0}{2} + \epsilon,
\end{align*}
so that we directly obtain that with probability at least $1-\delta$, 
\begin{align*}
  T(\delta) \leq \UpperSample_2(\mu,\delta) \defn \tau_0(\tfrac{(\mu_{\realbest}-\epsilon)+\mu_0}{2},0) + 3 \sum_{i=0}^K \tau_i( \max_{j \neq \realbest}\tfrac{\mu_{\realbest}+\mu_j}{2}, \min\{\tfrac{2}{7}\epsilon,\tfrac{\mu_{\realbest}-(\mu_0+\epsilon)}{5}\}).
  \end{align*}
  
Let us now simplify the expression.
It is straightforward to verify that the effective gap associated with arm $0$ is equal to
\begin{align*}
\effgap_0& \gtrsim  \min\left\{ \tfrac{\mu_{\realbest}- (\mu_0 +  \epsilon)}{2}, \max\left\{ \max_{j \neq \realbest}\tfrac{\mu_{\realbest}+\mu_j}{2} - \mu_0, \tfrac{2}{7} \epsilon \right\} \right\} \\
&\gtrsim \min\left\{ \mu_{\realbest} - (\mu_0+\epsilon), \max\{ \Delta_0, \frac{4}{7}\epsilon \} \right\} 
\end{align*}
and the effective gap for any other arm $i$ is equal to
\begin{align*}
\effgap_i &=  \max\left\{ |\max_{j \neq \realbest}\tfrac{\mu_{\realbest}+\mu_j}{2} - \mu_i| , \min\{\tfrac{2}{7}\epsilon,\tfrac{\mu_{\realbest} - (\mu_0 + \epsilon)}{5}\}  \right\} \\
&\gtrsim  \max\left\{ \Delta_i ,  \min\left\{ \mu_{\realbest} - (\mu_0+\epsilon), \epsilon \right\}  \right\}
\end{align*}
where we recall that $\Delta_i = \mu_{\realbest}-\mu_i$ if
$i \neq \realbest$, and
$\Delta_{\realbest} = \mu_{\realbest} - \max_{j \neq \realbest} \mu_j$
otherwise.
Using these quantities, the upper bound $\UpperSample_2(\mu,\delta)$ on the stopping time $T(\delta)$ scales like $\sum_{i=0}^K \effgap_i^{-2} \log(K \log(\effgap_i^{-2})/\delta)$. This concludes the proof of the proposition.

\subsection{Proof of Claim~\ref{ClaimKevin}}
\label{SecClaimProof}

\noindent Let
$\barmu \in [ \max_{j \neq \realbest} \mu_j, \mu_{\realbest} ]$ and
$\tau_i := \tau_i( \barmu, u\epsilon)$. 
The following result is a 
a key ingredient for the proof of the claim.

\begin{props} \label{prop:if_h_is_best_alot}
For any time $t$ and $u \leq 1/2$,
\begin{align*}
&\Big\{ |\{ s \leq t: h_s = \realbest\}| \geq \sum_{i=0}^K \tau_i \Big\} \\
&\implies  \{ \UCB_{\ell_t}(t) \leq  \barmu + u \epsilon \} \cap \{ \LCB_{h_t}(t) \geq \barmu - u \epsilon \} \\
&\implies \{ \LCB_{h_t}(t) - \UCB_{\ell_t}(t) \geq -\epsilon \} .
\end{align*}
\end{props}
\begin{proof}
If $h_s = \realbest$ then {\em some} $i \neq \realbest$ is assigned to $\ell_s$ and $\UCB_{i}(s) \leq  \max\{ \barmu , \mu_i + u\epsilon\}\leq \barmu + u\epsilon$ whenever $n_{i}(s) \geq \tau_i( \barmu, u\epsilon)$.
Because $\ell_s$ is the highest upper confidence bound, the sum over all $\tau_i$ represents exhausting all arms (i.e., pigeonhole principle).
An analogous result holds for $\LCB_{\realbest}(t)$.
\end{proof}

A direct consequence of Proposition~\ref{prop:if_h_is_best_alot} is
that even though we don't know which arm will be assigned to $h_t$ at
any given time $t$, we do know that if $h_t=\realbest$ for a
sufficient number of times, namely $\sum_{i=0}^K \tau_{i}$ times, the
termination criteria will be met.  Thus, assume
$h_t \neq \realbest$ and note that
\begin{align*}
\{ h_t = i, \ &  \mu_i < \mu_{\realbest} - \tfrac{5}{2}u \epsilon, \ \widehat{\mu}_{i,n_i(t)} \geq \min\{ \barmu, \mu_{\realbest} - \tfrac{3}{2}u\epsilon\} \} \\
&\implies \min\{ \barmu, \mu_{\realbest} - \tfrac{3}{2}u\epsilon\} \leq \widehat{\mu}_{i,n_i(t)} \leq \mu_i + \dev_{n_i(t)}(\tfrac{\delta}{2K})\\
&\implies \{ n_i(t) < \tau_i \}  
\end{align*}
where the last line follows from
$\mu_i + \dev_{n_i(t)}(\tfrac{\delta}{2K}) < \min\{ \barmu, \mu_i +
u\epsilon\} \leq \min\{ \barmu, \mu_{\realbest} -
\tfrac{3}{2}u\epsilon\}$ whenever $n_i(t) \geq \tau_i$.  Furthermore,
the following Proposition~\ref{prop:ht_is_high}, says for
$t \geq 2\sum_{i=0}^K \tau_i$ we have that
$\widehat{\mu}_{h_t,n_{h_t}(t)} \geq \min\{ \barmu, \mu_{\realbest} -
\tfrac{3}{2}u\epsilon\}$.
\begin{props} \label{prop:ht_is_high}
For any time $t$,
\begin{align*}
\{ t \geq 2\sum_{i=0}^K \tau_i \} \implies \{ \widehat{\mu}_{h_t,n_{h_t}(t)} \geq \min\{ \barmu, \mu_{\realbest} - \tfrac{3}{2}u\epsilon\} \}.
\end{align*}
\end{props}
The proof of the proposition can be found in Section~\ref{SecProofProp}.

Combining this fact with the display immediately above and the observation that some $i=h_t$, we have that $|\{ s \geq 2\sum_{i=0}^K \tau_i : \mu_{\realbest} - \mu_{h_s} \geq \tfrac{5}{2} u \epsilon \}| < \sum_{i=0}^K \tau_i$.
Now, on one of these times $t$ such that $\{h_t = i, n_i(t) \geq \tau_i, \mu_{\realbest} - \mu_i < \tfrac{5}{2} u \epsilon\}$, we have
\begin{align*}
\LCB_{i}(t) = \widehat{\mu}_{i,n_i(t)} - \dev_{n_i(t)}(\tfrac{\delta}{2K}) \geq \mu_i - 2\dev_{n_i(t)}(\tfrac{\delta}{2K}) \geq \min\{ \barmu,\mu_i -u\epsilon\} \geq \mu_{\realbest} - \tfrac{5}{2}u\epsilon .
\end{align*}

The above display with the next proposition completes the proof of Equation~\ref{eqn:prelim_stopping_condition}.
\begin{props}
For any time $t$,
\begin{align*}
\{ t \geq \sum_{i=0}^K \tau_i \} \implies \{ \max_{i=0,1,\dots,K} \UCB_i(t) \leq \mu_{\realbest} + u \epsilon \}.
\end{align*}
\end{props}
\begin{proof} 
Note that
\begin{align*}
\{ \UCB_i(t) \geq \mu_{\realbest} + u\epsilon \} &\implies \{ \mu_{\realbest} + u\epsilon \leq \UCB_i(t) = \widehat{\mu}_{i,n_i(t)} + \dev_{n_i(t)}(\tfrac{\delta}{2}) \leq \mu_i + 2\dev_{n_i(t)}(\tfrac{\delta}{2K}) \} \\
&\implies \{ n_i(t) < \tau_i \}
\end{align*}
since $\mu_i + 2\dev_{n_i(t)}(\tfrac{\delta}{2K}) < \max\{ \barmu, \mu_i + u \epsilon \} \leq \mu_{\realbest} + u\epsilon$ whenever $n_i(t) \geq \tau_i$.
Now, because at each time $t$, the arm $\arg\max_{j=0,1,\dots,K} \UCB_j(t)$ is pulled because it is either $h_t$ or $\ell_t$, we conclude that this arm can only be pulled $\tau_i$ times before satisfying $\UCB_i(t) \leq \mu_{\realbest} + u \epsilon$.  
\end{proof}

\subsection{Proof of Proposition~\ref{prop:ht_is_high}}
\label{SecProofProp}
The above proposition implies, 
\begin{align*}
\{ t \geq 2\sum_{i=0}^K \tau_i \} \implies \left\{ |\{ s \leq t: h_s \neq \realbest\}| \geq \sum_{i=0}^K \tau_i \right\}.
\end{align*}
\noindent Now consider the event
\begin{align*}
 \{h_t \neq \realbest, \ell_t = i \} &\implies \mu_{\realbest} \leq  \widehat{\mu}_{\realbest,n_{\realbest}(t)} + \dev_{n_{\realbest}(t)}(\tfrac{\delta}{2})  \leq  \widehat{\mu}_{i,n_i(t)} + \dev_{n_i(t)}(\tfrac{\delta}{2}) \leq \mu_i + 2\dev_{n_i(t)}(\tfrac{\delta}{2K}) \\
 &\implies \{ \mu_{\realbest} - \mu_i \leq 2\dev_{n_i(t)}(\tfrac{\delta}{2K}) \} \\
 &\implies \{ n_{i}(t) < \tau_i \} \cup \{ n_{i}(t) \geq \tau_i,  \mu_{\realbest} - \mu_i \leq 2\dev_{n_i(t)}(\tfrac{\delta}{2K}) \} \\
 &\implies \{ n_{i}(t) < \tau_i \} \cup \{ n_{i}(t) \geq \tau_i,  \mu_{\realbest} - \mu_i \leq \max\{ |\barmu - \mu_i|, u\epsilon \} \} \\
 &\implies \{ n_{i}(t) < \tau_i \} \cup \{ n_{i}(t) \geq \tau_i,  \mu_{\realbest} - \mu_i < u \epsilon \} \cup \{ n_{i}(t) \geq \tau_i,  i = \realbest\}
\end{align*}
by the definition of $\tau_i$.
Because at each time $s \leq t$ we have that {\em some} $i=\ell_s$, if $ |\{ s \leq t: h_s \neq \realbest\}| \geq \sum_{i=0}^K \tau_i$, we have that
\begin{align*}
\{ t \geq 2\sum_{i=0}^K \tau_i \} \implies \{ \exists i : n_{i}(t) \geq \tau_i \text{ and }  \mu_{\realbest} - \mu_i < u \epsilon \} \cup \{ n_{i}(t) \geq \tau_i \text{ and }  i = \realbest\}.
\end{align*}
We use the fact that such an $\ell_t = i \neq \realbest$ exists that satisfies $\mu_{\realbest} - \mu_i < u \epsilon$ to say
\begin{align*}
\exists i \neq \realbest: 
\widehat{\mu}_{i,n_i(t)} \geq \mu_i - \dev_{n_i(t)}(\tfrac{\delta}{2K}) \geq \mu_{i} - \max\{ \mu_{\realbest}-\mu_i , u\epsilon\}/2 \geq \mu_{\realbest}- \tfrac{3}{2} u \epsilon
\end{align*}
or $\ell_t = \realbest$ and 
\begin{align*}
\widehat{\mu}_{\realbest,n_{\realbest}(t)} \geq  \mu_{\realbest} - \dev_{n_{\realbest}(t)}(\tfrac{\delta}{2K}) \geq \mu_{\realbest} - \max\{ \mu_{\realbest} - \barmu, u\epsilon\} /2 = \min\{ \barmu, \mu_{\realbest} - \tfrac{1}{2} u \epsilon \}. 
\end{align*}
Because $\widehat{\mu}_{h_t,n_{h_t}(t)} \geq \max_{i=0,1,\dots,K}\widehat{\mu}_{i,n_{i}(t)}$, the proof of the claim is complete.

\end{document}